\documentclass{article}

\usepackage{microtype}
\usepackage{graphicx}
\usepackage{subcaption}
\usepackage{booktabs} % for professional tables

\usepackage{hyperref}

\usepackage[preprint]{icml2026}

\usepackage{amsmath}
\usepackage{amssymb}
\usepackage{mathtools}
\usepackage{amsthm}
\usepackage{multirow}

\usepackage[capitalize,noabbrev]{cleveref}

\theoremstyle{plain}
\newtheorem{theorem}{Theorem}[section]
\newtheorem{proposition}[theorem]{Proposition}

\theoremstyle{definition}

\theoremstyle{remark}
\newtheorem{remark}[theorem]{Remark}

\usepackage[textsize=tiny]{todonotes}

\icmltitlerunning{Submission and Formatting Instructions for ICML 2026}

\begin{document}

\twocolumn[
  \icmltitle{S$^3$-Attention:Attention-Aligned Endogenous Retrieval for Memory-Bounded Long-Context Inference}



  \icmlsetsymbol{equal}{*}

  \begin{icmlauthorlist}
    \icmlauthor{Qingsen Ma}{equal,bupt}
    \icmlauthor{Dianyun Wang}{equal,bupt}
    \icmlauthor{Yaoye Wang}{equal,bupt}
    \icmlauthor{Lechen Ning}{bupt}
    \icmlauthor{Sujie Zhu}{bupt}
    \icmlauthor{Xiaohang Zhang}{bupt}
    \icmlauthor{Jiaming Lyu}{bupt}

    \icmlauthor{Linhao Ren}{bupt}
    \icmlauthor{Zhenbo Xu}{bupt}
    \icmlauthor{Zhaofeng He}{bupt}
  \end{icmlauthorlist}

  \icmlaffiliation{bupt}{Beijing University of Posts and Telecommunications, Beijing, China}

  \icmlcorrespondingauthor{Zhenbo Xu}{xuzhenbo@bupt.edu.cn}
  \icmlcorrespondingauthor{Zhaofeng He}{zhaofenghe@bupt.edu.cn}

  \icmlkeywords{Machine Learning, ICML}

  \vskip 0.3in
]



\printAffiliationsAndNotice{}  % no special notice (required even if empty)

\begin{abstract}
% Large language models are increasingly used on multi-document and long-form inputs, yet long-context inference remains memory- and noise-inefficient: KV caching scales linearly with context length, and external retrievers may return lexically similar but causally irrelevant passages. We present S$^3$-Attention, a memory-first inference-time framework that treats long-context usage as attention-aligned endogenous retrieval. S$^3$-Attention decodes transient key/query projections into Top-k sparse feature IDs via lightweight sparse autoencoders and builds a CPU inverted index (feature → token positions/spans) during a streaming scan, allowing the KV cache to be discarded and keeping GPU memory bounded by the scan chunk size. At generation time, feature co-activation retrieves compact evidence spans, optionally fused with BM25 for exact matching. Under a unified LongBench evaluation protocol (fixed prompting/decoding and matched token budgets), S$^3$-Hybrid closely matches full-context inference on multiple model families and improves robustness on several information-dense settings. 
Long-context inference in Large Language Models (LLMs) faces a critical dilemma: maintaining full context incurs linear KV cache scaling, while offloading to external retrievers often yields lexically similar but causally irrelevant passages. To bridge this gap, we present $S^3$-Attention, a framework that transforms memory-bound inference into a streaming, attention-aligned endogenous retrieval process. Our approach distinguishes itself by decoding transient attention states into Top-$k$ sparse feature IDs via lightweight sparse autoencoders. Instead of maintaining a massive GPU Key-Value cache, we build a CPU-based inverted index during a streaming scan, \textbf{ensuring GPU memory remains constant and bounded by chunk size}. This mechanism aligns retrieval directly with the model's inherent reasoning patterns, using feature co-activation (optionally fused with BM25) to recall compact evidence spans. Empirically, under the unified LongBench protocol, $S^3$-Hybrid closely matches full-context performance across multiple model families and improves robustness in information-dense settings by effectively filtering noise.
\end{abstract}

\section{Introduction}

Large language models (LLMs) have evolved into the default interface for complex cognitive tasks, transitioning from processing single documents to digesting massive collections of reports, codebases, and conversational histories. This shift toward contexts that far exceed a single attention window is not merely a technical scaling challenge but a fundamental requirement for the next generation of AI systems~\cite{ding2024longrope,zhao2024longagent}. However, bridging the gap between the theoretical capability of processing million-token regimes and the practical reality of efficient deployment remains a critical bottleneck.

The prevailing response has been to extend context lengths through continued training or architectural modifications~\cite{liu2025comprehensive,chen2024core,hu2024efficient,du2025long,mao2025lift,mohtashami2023landmark}. Yet, longer windows do not automatically translate into reliable reasoning. Full-context inference is deployment-unfriendly~\cite{ma2024compressing}: self-attention incurs quadratic compute costs~\cite{lou2024sparser}, and Key-Value (KV) caches scale linearly~\cite{sun2024shadowkv} with input length~\cite{tang2024quest}, quickly saturating GPU memory~\cite{wang2025llms}. Furthermore, real-world long inputs are typically sparse in signal---only a small fraction of tokens are causally useful~\cite{zhu2025tactic}---so naively attending to everything often dilutes evidence and amplifies distraction~\cite{xu2024recycled,hooper2025multipole,zarch2025delta}. The pragmatic alternative, Retrieval-Augmented Generation (RAG)~\cite{lewis2020retrieval,cheng2025survey}, solves the memory issue but introduces a ``semantic mismatch.'' Because external retrievers operate in an embedding space independent of the generator's internal reasoning features~\cite{wei2025alignrag}, they often retrieve lexically similar but causally irrelevant text, degrading multi-hop reasoning~\cite{liu2025hoprag} and increasing hallucinations~\cite{wang2025rag+}.

This dilemma prompts a fundamental question: \textbf{Can we achieve the memory efficiency of RAG while retaining the cognitive alignment of full-context attention?} This motivates a shift toward \emph{endogenous retrieval}, where the model retrieves evidence using its own internal signals~\cite{wu2024retrieval,zhang2025query}. While recent methods like InfiniRetri~\cite{ye2025infinite} and various KV-compression techniques~\cite{xiao2024duoattention} utilize attention patterns to locate information, they lack a scalable mechanism for indexing. Directly operating on dense attention weights or cached states does not yield a viable queryable memory~\cite{chen2025retroinfer}: these continuous, high-dimensional representations are too expensive to store or search at token granularity~\cite{ma2024compressing,liu2024retrievalattention,liu2025chunkkv}. Existing systems effectively compress memory but fail to provide an explicit, searchable index that allows for precise evidence retrieval.

To realize this vision, we propose S$^3$-Attention (Sparse \& Semantic Streaming Attention), a framework designed specifically for settings where GPU memory is the primary bottleneck and causal evidence is sparse. S$^3$-Attention transforms long-context inference into a streaming, attention-aligned retrieval procedure. The core challenge is converting the model's dense internal states into a format that is both lightweight and efficiently searchable. We address this by training lightweight Top-$k$ sparse autoencoders (SAEs) to compress the transient \emph{key} and \emph{query} projections---which are already computed by the model---into discrete sparse semantic features. During a single streaming prefill pass, S$^3$-Attention builds an inverted index on the CPU and immediately discards the KV cache, effectively achieving $O(1)$ GPU memory with respect to total context length. At generation time, the query's SAE features retrieve high-density spans via feature co-activation. To ensure robustness against rare entities that may elude semantic compression, we optionally fuse this endogenous signal with BM25 lexical matching, yielding S$^3$-Hybrid.

Empirically, S$^3$-Attention demonstrates exceptional fidelity on the LongBench suite. S$^3$-Hybrid retains 99.4\% of full-context performance on Llama-3-8B (24.87 vs.\ 25.01) and over 99\% on Qwen2-7B, while enabling constant-GPU-memory processing. Notably, we observe a ``denoising'' effect on information-dense tasks, where our selective processing filters distraction and occasionally outperforms full-context baselines.

\paragraph{Contributions.} This paper makes three contributions:
\begin{itemize}
    \item We articulate long-context inference as an \emph{endogenous retrieval} problem and motivate why aligning retrieval with internal attention representations mitigates the semantic gap of exogenous retrievers.
    \item We introduce S$^3$-Attention, which utilizes SAE-decoded sparse semantic features to build a streaming inverted index, achieving $O(1)$ GPU memory without fine-tuning the base LLM.
    \item We demonstrate that fusing endogenous semantic signals with BM25 yields a robust hybrid retriever with near-lossless fidelity to full-context inference, validated across multiple model families on LongBench.
\end{itemize}

\section{Related Works}

Long-context reasoning with large language models (LLMs) involves a fundamental trade-off between fidelity, efficiency, and robustness. Feeding the full context avoids information loss but incurs prohibitive computation and KV-cache overhead, while irrelevant or noisy tokens increasingly dilute causally relevant evidence as context length grows. Prior work shows that only a small subset of past tokens materially contributes to the current prediction, motivating query-aware selection or sparsification in long-context settings \cite{hu2024efficient,tang2024quest}. Existing methods can be broadly categorized into approaches based on exogenous versus endogenous signals.

\textbf{Exogenous signals:}
Retrieval-Augmented Generation (RAG) is the canonical exogenous approach, where an external retriever selects passages deemed relevant to the query \cite{lewis2020retrieval}. Subsequent work explores system-level improvements such as reranking, query decomposition, and speculative decoding \cite{li2025lara,liu2025poqd,yang2025re}. However, RAG is highly sensitive to retrieval quality: in long or noisy contexts it may miss critical evidence or select spurious passages \cite{xian2024vulnerability,chen2024benchmarking}. More fundamentally, relevance defined by external similarity does not necessarily align with the information the LLM internally uses for generation, leading to implicit misalignment \cite{jin2025llm}. Reflection-based variants partially alleviate retrieval errors \cite{asai2024self,chen2025c}, but remain dependent on external control signals.

\textbf{Endogenous signals:}
Another line of work studies internal mechanisms of LLMs. In-context learning has been formalized as conditioned associative memory retrieval \cite{smart2025context}, and empirical analyses show that long-context localization often relies on a small number of retrieval-oriented attention heads \cite{zhao2024understanding}. Query-aware sparsification and causal retrieval further indicate that only a limited subset of tokens affects prediction outcomes \cite{hu2024efficient,tang2024quest}. Endogenous signals have also been used for KV-cache selection or compression to improve inference efficiency \cite{zhang2023h2o,li2024snapkv,wu2025scope}, but these methods primarily target system efficiency and do not explicitly identify interpretable, causally relevant evidence segments for answering long-context questions \cite{huang2025internal}.

\begin{figure*}[t]
\centering
\includegraphics[width=\textwidth]{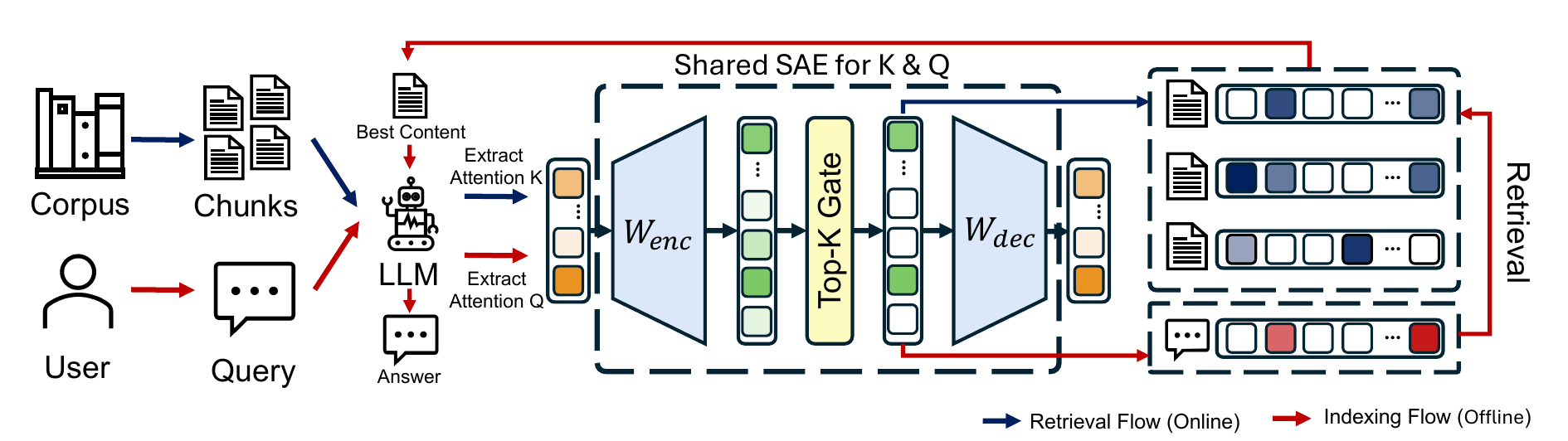}
\caption{\textbf{Overview of the $S^3$-Attention framework.}
The framework consists of two phases connected by a Top-$k$ Sparse Autoencoder (SAE).
\emph{Streaming Semantic Indexing} (red flow) encodes transient key projections into sparse semantic features to build a CPU-based inverted index, enabling the dense KV cache to be discarded and maintaining an $O(1)$ GPU memory footprint.
\emph{Endogenous Retrieval} (blue flow) encodes query projections with the same SAE, where activated features retrieve relevant context spans via feature co-activation, which are then fed back into the LLM for answer generation.SAE is trained on K projections and reused to discretize Q.}

\label{fig:overview}
\end{figure*}

\section{Methodology}

\subsection{Overview: From Exogenous to Endogenous Retrieval}

Handling long contexts in Large Language Models (LLMs) presents a trilemma among fidelity (Full-Attention), efficiency (RAG), and robustness (Noise Tolerance). We argue that the noise sensitivity observed in standard Retrieval-Augmented Generation (RAG) stems from a fundamental \textit{Semantic Gap}: the reliance on \textbf{exogenous} retrievers (e.g., BERT-based embeddings) whose latent spaces are misaligned with the reasoning heads of the generative LLM. This heterogeneity often leads to the retrieval of superficially similar but causally irrelevant segments, injecting noise that hallucinates the generation.

Drawing inspiration from \textbf{human cognitive processes}—specifically how readers engage with information-dense texts (e.g., news reports or technical manuals)—we observe that humans do not process every word with equal weight. Instead, they employ a \textit{goal-oriented selective attention mechanism}, scanning the text to lock onto salient information relevant to their query while actively filtering out background noise.

To replicate this endogenous process in LLMs, we propose \textbf{S\textsuperscript{3}-Attention} (Sparse \& Semantic Streaming Attention). Bridging the gap between memory constraints and context understanding, S\textsuperscript{3}-Attention replaces the external retriever with an \textbf{Endogenous Retrieval} mechanism. By decoding the model's own transient Key and Query states into sparse semantic features, we enable the model to perform ``introspection''—explicitly retaining only the context segments that trigger its own attention mechanisms. This approach ensures that the retrieved context is cognitively aligned with the generation process.

\subsection{Theoretical Framework: The Endogenous Information Bottleneck}

% We formulate long-context compression as an \textbf{Information Bottleneck (IB)} optimization problem.We assume that the model’s internal attention patterns serve as a sufficient statistic for the dependency between context and future generations. Let $C$ be the full context, $Q$ be the query, and $Y$ be the target answer generated by the LLM. Our goal is to find a compressed representation $\hat{C}$ (where $|\hat{C}| \ll |C|$) that solves:

% \begin{equation}
%     \max_{\hat{C}} \; I(Y; \hat{C} \mid Q) - \beta \cdot I(C; \hat{C})
%     \label{eq:ib_objective}
% \end{equation}

% $\beta$  is implicitly controlled by sparsity level, feature frequency thresholding, and spatial suppression.
% where the first term encourages \textit{sufficiency} (retaining answer-relevant information) and the second term enforces \textit{compression} (minimizing redundancy).

We view long-context compression as selecting a compact subset of the input that preserves task-relevant evidence for answering a query.
Under this perspective, an ideal objective can be formulated using mutual information, for example by maximizing
\[
I(Y;\,\hat{C} \mid Q)
\]
subject to a compression constraint on $\hat{C}$.
However, this quantity is intractable to estimate for modern large language models, and we do not attempt to compute it in practice.

Instead, we propose an operational \emph{endogenous} proxy derived from the model's own attention-matching signals.
Specifically, we discretize transient attention projections into a  sparse feature space and rank context positions
by their inverse-document-frequency (IDF) weighted feature co-activation with the query.
Appendix~\ref{app:theory} provides a motivating inequality under explicit simplifying assumptions.
We emphasize that this result should be interpreted as a heuristic justification of our scoring rule,
rather than as a guarantee that feature overlap faithfully tracks mutual information in real models.

\textbf{The Exogenous Gap.} Standard RAG methods optimize a surrogate objective $\max \text{Sim}(\text{Embed}(c), \text{Embed}(Q))$ using external encoders (e.g., BGE, Contriever). However the external embedding space does not capture the LLM's internal notion of ``relevance to $Y$''.

\textbf{Our Insight.} We observe that the attention mechanism \textit{implicitly} solves a related problem. Let $\mathbf{A} = \text{softmax}(\mathbf{Q}\mathbf{K}^\top / \sqrt{d})$ be the attention matrix. High attention weights $A_{ij}$ indicate that token $j$ in the context is \textit{causally useful} for predicting the next token at position $i$. By extracting the LLM's own Key-Query matching patterns, we obtain an \textbf{endogenous} relevance signal that is inherently aligned with the generation process.

\textbf{From Attention to Sparse Features.} Direct use of attention weights is infeasible due to quadratic complexity. Instead, we leverage Sparse Autoencoders (SAEs) to decompose the dense Key/Query vectors into interpretable sparse features. We show in Appendix~\ref{app:theory} that under strong simplifying assumptions, feature co-occurrence between $\mathbf{K}$ and $\mathbf{Q}$ provides a tractable lower bound on mutual information:
\begin{equation}
    I(Y; \hat{C} \mid Q) \geq \mathbb{E}\left[ \sum_{t \in \hat{C}} \sum_{f \in \mathcal{F}_Q} \mathbb{1}[f \in \mathcal{F}_t] \cdot w_f \right] + \text{const}
    \label{eq:mi_bound}
\end{equation}
where $\mathcal{F}_Q$ and $\mathcal{F}_t$ are the active SAE features for the query and context token $t$, respectively. This directly motivates our scoring function in Eq.~\ref{eq:scoring}.

\subsection{Deciphering Attention via Sparse Autoencoders}

The foundation of our approach is the ability to interpret the polysemantic activation vectors within the LLM's attention heads. We employ \textbf{Top-K Sparse Autoencoders (SAEs)} \cite{gao2024scaling} to decompose these dense vectors into interpretable sparse features.

Let $\mathbf{x} \in \mathbb{R}^{d_{head}}$ be the activation vector projected by the Key ($\mathbf{K}$) or Query ($\mathbf{Q}$) matrix in layer $\ell$. An SAE consists of an encoder and a decoder. The encoder projects $\mathbf{x}$ into a higher-dimensional latent space $\mathbb{R}^{d_{latent}}$ (where $d_{latent} \gg d_{head}$):
\begin{equation}
    \mathbf{z} = \text{ReLU}(\mathbf{W}_{enc}(\mathbf{x} - \mathbf{b}_{dec}) + \mathbf{b}_{enc})
\end{equation}
To enforce sparsity, we apply a Top-K nonlinearity, retaining only the $k$ most active latent features:
\begin{equation}
    \hat{\mathbf{z}} = \text{TopK}(\mathbf{z}, k), \quad \|\hat{\mathbf{z}}\|_0 = k
\end{equation}
The original activation is reconstructed as $\hat{\mathbf{x}} = \hat{\mathbf{z}}\mathbf{W}_{dec} + \mathbf{b}_{dec}$. We train a SAE on key projections and reuse the same SAE to encode query projections, ensuring a  feature-ID space.

\subsection{Phase 1: Streaming Semantic Indexing}

To process Massive contexts without maintaining a linear-growth KV cache, we introduce a streaming indexing protocol.

We process the input context $\mathcal{C}$ in sequential chunks $\{c_1, \dots, c_m\}$. For each chunk, the LLM performs a forward pass to generate Key states. Instead of caching these tensors, we immediately encode them via the SAEs into sparse indices:
\begin{equation}
    \mathcal{F}_{t}^{(\ell)} = \text{Indices}(\text{SAE}_{\ell}(\mathbf{k}_{t}^{(\ell)}))
\end{equation}
where $\mathcal{F}_{t}^{(\ell)}$ represents the set of active semantic features for the token at position $t$ in layer $\ell$.

We construct a lightweight \textbf{Inverted Semantic Index} $\mathcal{I}$ on the CPU, mapping each feature ID to a list of absolute positions:
\begin{equation}
    \forall f \in \mathcal{F}_{t}^{(\ell)}, \quad \mathcal{I}_{\ell}[f] \leftarrow \mathcal{I}_{\ell}[f] \cup \{t\}
\end{equation}
Crucially, once the features are indexed, the GPU memory for activations and KV cache is released. This transforms the memory complexity of the pre-filling phase from $\mathcal{O}(L)$ to $\mathcal{O}(1)$, limited only by chunk size.

\paragraph{CPU Index Size.}
S$^3$-Attention achieves \( O(1) \) GPU memory with respect to the total context length by discarding the KV cache after feature decoding, but it still maintains a CPU-side inverted index.
Let \( L \) denote the number of context tokens, \( |L_{\text{target}}| \) the number of instrumented layers, and \( k \) the Top-\(k\) sparsity per token after head aggregation.
The total number of postings is
\[
P = L \cdot |L_{\text{target}}| \cdot k .
\]
In an idealized implementation that stores token positions as \texttt{int32}, this corresponds to approximately \( 4P \) bytes for positions alone
(e.g., \( L = 128\text{K},\ |L_{\text{target}}| = 4,\ k = 128 \Rightarrow \) about \(256\,\text{MiB}\)).
Our current prototype uses Python \texttt{dict}/\texttt{list} posting lists, which incur substantial overhead; production implementations should instead use contiguous integer arrays with delta encoding and optional stop-feature pruning to reduce constant factors.

\subsection{Phase 2: Semantic Density Estimation \& Retrieval}

Upon receiving a query $Q$, our goal is to identify regions in the context that maximizes the \textit{expected attention} of the LLM.

\textbf{Query Decoding.} We compute the Query projections for $Q$ and encode them using the same SAEs used for indexing. This extracts the model's intrinsic ``search intent'':
\begin{equation}
    \{ (w_q, f_q) \} = \text{SAE}_{\ell}(\mathbf{q}^{(\ell)})
\end{equation}
where $w_q$ is the activation strength of feature $f_q$.

\textbf{Semantic Density Estimation.} We calculate a semantic relevance score $S[t]$ for every position $t$ in the context. Unlike vector similarity search which treats text as static blocks, we formulate this as a feature voting process:
\begin{equation}
    S[t] = \sum_{\ell \in \mathcal{L}_{target}} \sum_{f \in f_q^{(\ell)}} \mathbb{1}(t \in \mathcal{I}_{\ell}[f]) \cdot w_{q,f}^{(\ell)} \cdot \text{IDF}(f)
    \label{eq:scoring}
\end{equation}
Here, $\text{IDF}(f)$ down-weights high-frequency features (e.g., common syntactic patterns) to focus on rare, information-rich concepts.

\textbf{Adaptive Granularity.} Fixed-size chunking (used in RAG) often fragments semantic units. S\textsuperscript{3}-Attention employs a dynamic approach. We apply a 1D convolution kernel to smooth the sparse score array $S$, generating a \textit{semantic density curve}. We then perform Non-Maximum Suppression (NMS) to identify peak density regions. For each peak, we retrieve a variable-length span, ensuring the context is cut at natural semantic boundaries rather than arbitrary token counts.

% \textbf{Semantic Density Estimation.} We calculate a semantic relevance score $S[t]$ for every position $t$ in the context. Motivated by the lower bound in Eq.~\ref{eq:mi_bound}, we formulate this as a weighted feature voting process:
% \begin{equation}
%     S[t] = \sum_{\ell \in \mathcal{L}_{target}} \sum_{f \in \mathcal{F}_Q^{(\ell)}} \mathbb{1}[f \in \mathcal{F}_t^{(\ell)}] \cdot w_{Q,f}^{(\ell)} \cdot \text{IDF}(f)
%     \label{eq:scoring}
% \end{equation}
% Here, $\text{IDF}(f) = 1/(\log(1 + \text{freq}(f)) + 1)$ down-weights high-frequency features (e.g., common syntactic patterns) to focus on rare, information-rich concepts. This can be viewed as a sparse approximation to the KL divergence term in IB (see Appendix~\ref{app:idf}).

\subsection{Phase 3: Hybrid Fusion}

To guarantee robustness across diverse tasks, we fuse the endogenous semantic signal with structural and lexical priors:
\begin{equation}
    \mathcal{M}_{final} = \mathcal{M}_{S^3} \cup \mathcal{M}_{BM25} \cup \mathcal{M}_{Bias}
\end{equation}
\begin{itemize}
    \item $\mathcal{M}_{S^3}$: Indices selected by the SAE-driven endogenous retrieval, capturing abstract reasoning chains.
    \item $\mathcal{M}_{BM25}$: Indices selected by lexical matching, compensating for rare entities (e.g., random IDs) that SAEs might not reconstruct perfectly.
    \item $\mathcal{M}_{Bias}$: A positional bias retaining the first and last $N$ tokens (Lead/Tail), mitigating the ``Lost-in-the-Middle'' phenomenon.
\end{itemize}
The fused indices are used to gather the original tokens into a compressed context $\tilde{\mathcal{C}}$, which is fed to the LLM for generation.

\section{Experiments}
\label{sec:experiments}

%\subsection{Experimental Setup}

\subsection{Experimental Setup and Methodology}

To demonstrate the universality of S\textsuperscript{3}-Attention, we conduct experiments across three distinct state-of-the-art open-weights LLMs (Llama-3.1, Mistral, and Qwen2) using \textbf{bfloat16} precision. These models were selected to represent diverse attention mechanisms and positional embeddings. We evaluate performance using the \textbf{LongBench} suite, focusing on nine datasets spanning Single-Document QA, Multi-Document QA, and Summarization tasks. This selection rigorously tests the models' ability to retrieve and synthesize information from contexts exceeding 10k tokens.

Our pipeline consists of two stages: \textit{Offline SAE Training} and \textit{Online Inference}. For the offline stage, we train Top-K Sparse Autoencoders (SAEs) on general linguistic corpora to create a feature space for efficient discretization. For online inference, we benchmark S\textsuperscript{3}-Attention against strong baselines, including: (1) \textbf{Full-Context} (upper bound), (2) \textbf{Retrieval-Augmented Generation (RAG)}, and (3) KV cache compression methods such as H2O~\cite{zhang2023h2o} and StreamingLLM~\cite{xiao2024efficientstreaminglanguagemodels}. To ensure rigorous comparison, we employ a unified evaluation protocol with identical decoding settings and prompts across all methods.

Specific details regarding model versions, dataset breakdowns, SAE training configurations, and decoding parameters are provided in Appendix~\ref{app:exp_setups}.

\subsection{Qualitative Analysis: Mechanism of Endogenous Alignment}

To validate our hypothesis that endogenous retrieval eliminates the semantic gap inherent in exogenous methods, we conducted a microscopic analysis of the retrieval behaviors. Figure~\ref{fig:semantic_alignment} visualizes the retrieval scores of standard RAG (BGE-Small) against the semantic activation maps of S\textsuperscript{3}-Attention on a multi-hop reasoning query: \textit{``Which film starring Tom Hanks was directed by Steven Spielberg?''}

\paragraph{Visualizing the Semantic Gap.}
As shown in Figure~\ref{fig:semantic_alignment}, the contrast between the two paradigms explains the performance gap observed in our main experiments:

\begin{itemize}
    \item \textbf{Exogenous RAG Failure (Surface-Level Matching):} The RAG retriever (Top Panel) falls into a ``lexical trap.'' It assigns the highest similarity score ($0.751$) to Sentence 1 (green bar), which is a generic biography of Tom Hanks. While this sentence shares high lexical overlap with the query entities (``Tom Hanks''), it contributes zero informational value to the specific question. The actual answer (Sentence 5, ``The Post'') is ranked lower ($0.635$), burying the true signal under noise. This confirms that embedding models often prioritize \textit{topical similarity} over \textit{truthfulness}.

    \item \textbf{Endogenous S\textsuperscript{3} Success (Deep Semantic Anchoring):} In stark contrast, S\textsuperscript{3}-Attention (Bottom Panel) demonstrates an intrinsic ability to filter out noise. The SAE-decoded semantic activations are near-zero for the generic biography section. Instead, we observe sharp, high-confidence activation peaks (marked by red arrows) precisely at the tokens ``The Post'' and conceptually related terms like ``Pentagon Papers.'' This indicates that the LLM is not merely matching names; it is attending to the \textit{causal evidence} required to resolve the query. The mechanism effectively acts as a semantic band-pass filter, suppressing the ``Tom Hanks'' biography noise while amplifying the specific film entity.
\end{itemize}

This phenomenon is consistent across our dataset. S\textsuperscript{3}-Attention achieves a superior \textbf{Signal-to-Noise Ratio (SNR)}, activating only the 1--2\% of tokens that serve as reasoning bridges, whereas RAG retrieves broad, lexically dense but logically shallow chunks.

\begin{figure}[t]
    \centering
    \includegraphics[width=1.0\linewidth]{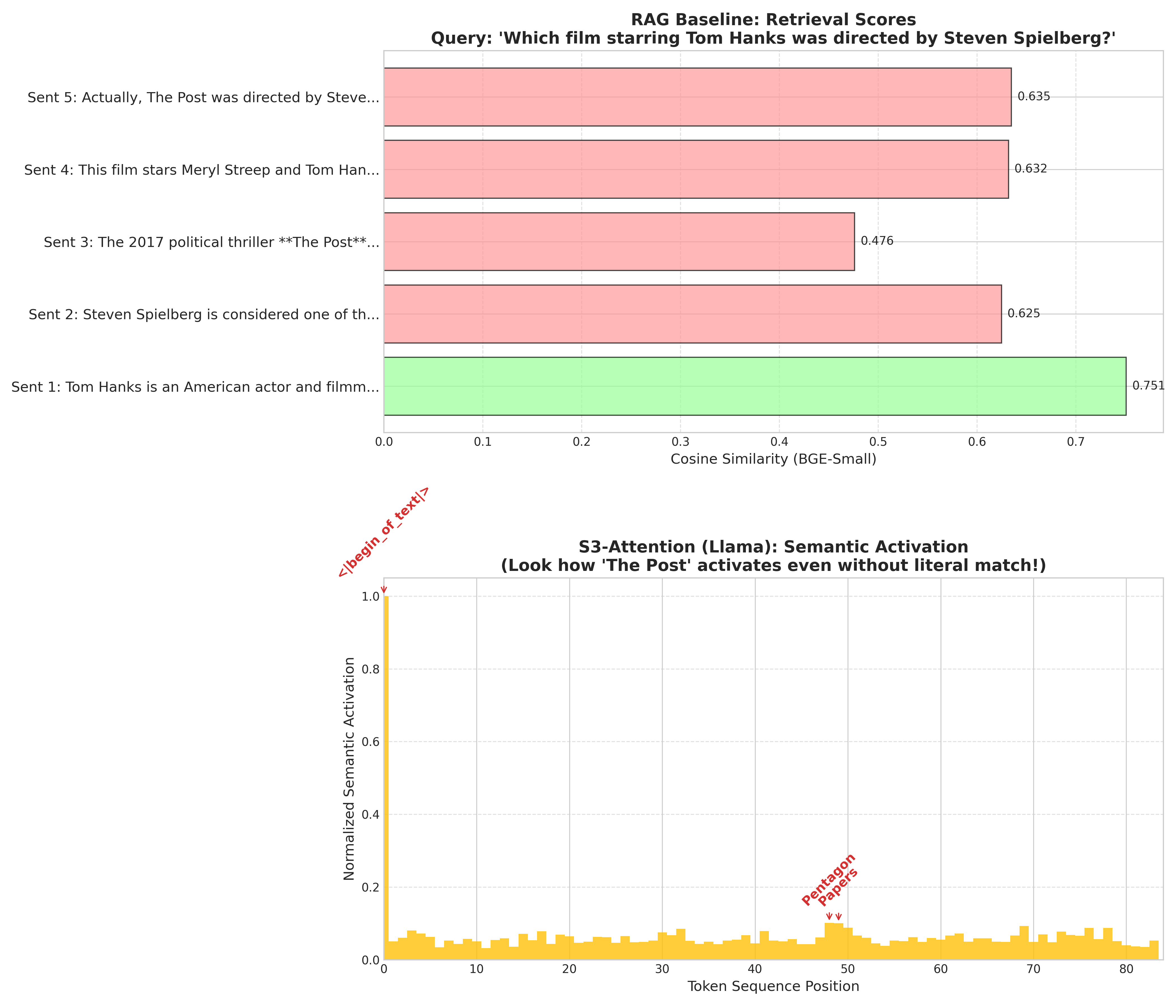}
    \caption{\textbf{Endogenous vs. Exogenous Retrieval.} \textbf{Top:} RAG (BGE-Small) is distracted by high lexical overlap... \textbf{Bottom:} S\textsuperscript{3}-Attention (Ours) ignores the distraction... (For a larger view, please refer to Figure~\ref{fig:semantic_alignment_large} in the Appendix.)}
    
    \label{fig:semantic_alignment}
\end{figure}
% \begin{figure}[t]
%     \centering
%     \includegraphics[width=1.0\linewidth]{semantic_alignment_Llama.png}
%     \caption{\textbf{Endogenous vs. Exogenous Retrieval.} \textbf{Top:} RAG (BGE-Small) is distracted by high lexical overlap, ranking a generic biography (Sentence 1) higher than the true answer (Sentence 5). \textbf{Bottom:} S\textsuperscript{3}-Attention (Ours) ignores the distraction, showing sparse activation peaks exclusively at the semantic answer anchor (``The Post'') and its reasoning evidence (``Pentagon Papers'').}
%     \label{fig:semantic_alignment}
% \end{figure}

Additional visualization results and analysis of other datasets are provided in \textbf{Appendix~\ref{app:extended_qual}}.

\begin{table*}[h]
\centering
\caption{\textbf{Comprehensive evaluation across Single-/Multi-Document QA and summarization benchmarks.}}
\label{tab:full_results_all_models}
\resizebox{\textwidth}{!}{%
\begin{tabular}{l|c|l|ccc|ccc|ccc|c}
\toprule
\multirow{2}{*}{\textbf{Model}} & \multirow{2}{*}{\textbf{Size}} & \multirow{2}{*}{\textbf{Method}} & \multicolumn{3}{c|}{\textbf{Single-Document QA}} & \multicolumn{3}{c|}{\textbf{Multi-Document QA}} & \multicolumn{3}{c|}{\textbf{Summarization}} & \multirow{2}{*}{\textbf{Avg}} \\
 &  &  & \textbf{NrtvQA} & \textbf{Qasper} & \textbf{MF-en} & \textbf{HotpotQA} & \textbf{2WikiMQA} & \textbf{Musique} & \textbf{GovReport} & \textbf{QMSum} & \textbf{MultiNews} &  \\
\midrule
\multirow{12}{*}{llama3.1-8B-Instruct}
 & 512 & StreamingLLM            & 19.03 & 12.78 & 28.67 & 37.83 & 29.97 & 16.55 & 20.30 & 20.94 & 24.56 & 23.40 \\
 & 512 & H2O                     & 22.84 & 16.80 & 32.36 & 41.43 & 34.07 & 19.30 & 22.28 & 22.81 & 23.69 & 26.62 \\
 & 512 & SnapKV                  & 24.62 & 22.78 & 37.88 & 42.96 & 34.82 & 20.65 & 22.63 & 22.54 & 23.93 & 28.42 \\
 & 512 & PyramidKV               & 24.48 & 23.51 & 36.14 & 42.33 & 31.95 & 20.73 & 23.37 & 23.01 & 24.37 & 27.99 \\
 & 512 & Dynamic                 & 24.78 & 24.76 & 36.84 & 44.13 & 33.25 & 20.82 & 23.00 & 22.76 & 24.14 & 28.50 \\
 &  -  & \textbf{BM25}           & 17.24 & 19.86 & 44.91 & 48.84 & 16.86 & 18.88 & 18.32 & 11.86 & 23.44 & 24.25 \\
 &  -  & \textbf{RAG}            & 21.08 & 21.43 & 44.15 & 49.31 & 17.98 & 20.61 & 19.29 & 11.09 & 23.41 & 25.04 \\
 &  -  & \textbf{S$^3$-Pure}  & 16.81 & 20.04 & 42.82 & 41.28 & 14.92 & 16.63 & 17.16 & 10.05 & 23.67 & 22.60 \\
 &  -  & \textbf{S$^3$-Hybrid}      & 22.28 & 21.50 & 43.45 & 47.07 & 17.87 & 18.69 & 19.33 & 10.21 & 23.41 & 24.87 \\
 &  -  & \textbf{FullKV}     & 23.81 & 20.56 & 44.23 & 49.74 & 18.61 & 19.45 & 19.57 &  9.77 & 23.33 & 25.01 \\
\midrule
\multirow{12}{*}{Qwen2-7B-Instruct}
 & 512 & StreamingLLM            & 20.47 & 26.97 & 32.64 & 14.31 & 14.39 &  6.82 & 25.70 & 19.31 & 24.88 & 20.61 \\
 & 512 & H2O                     & 22.88 & 34.28 & 41.40 & 13.30 & 14.60 &  8.31 & 23.69 & 22.07 & 22.72 & 22.81 \\
 & 512 & SnapKV                  & 23.86 & 38.61 & 44.65 & 15.60 & 14.62 &  9.13 & 24.56 & 22.39 & 23.07 & 24.50 \\
 & 512 & PyramidKV               & 24.47 & 37.60 & 43.51 & 14.48 & 12.83 &  8.99 & 23.59 & 22.30 & 22.41 & 23.91 \\
 & 512 & Dynamic                 & 24.66 & 40.44 & 45.30 & 15.42 & 13.89 &  8.46 & 25.51 & 22.77 & 22.92 & 24.82 \\
 &  -  & \textbf{BM25}           & 13.99 & 20.31 & 38.00 & 18.67 & 13.61 & 10.12 & 19.69 & 11.14 & 21.66 & 18.80 \\
 &  -  & \textbf{RAG}            & 13.97 & 20.48 & 38.63 & 19.04 & 13.82 & 11.82 & 19.69 & 11.33 & 21.73 & 19.17 \\
 &  -  & \textbf{S$^3$-Pure}  & 13.28 & 17.24 & 34.87 & 17.79 & 12.30 & 10.54 & 16.97 & 10.85 & 21.66 & 17.50 \\
 &  -  & \textbf{S$^3$-Hybrid}      & 14.76 & 20.04 & 36.75 & 17.71 & 12.88 & 11.91 & 19.69 & 10.79 & 21.66 & 18.80 \\
 &  -  & \textbf{FullKV}     & 13.22 & 20.46 & 37.13 & 19.30 & 14.07 & 11.14 & 20.69 & 10.53 & 21.73 & 18.92 \\
\midrule
\multirow{18}{*}{\shortstack[l]{Mistral-7B\\-Instruct-v0.3}}
 & 128  & StreamingLLM           & 16.91 & 21.51 & 24.85 & 34.14 & 26.99 & 16.64 & 15.67 & 18.61 & 14.40 & 21.19 \\
 & 128  & H2O                    & 21.25 & 26.66 & 35.13 & 38.82 & 29.80 & 18.88 & 21.00 & 19.50 & 18.63 & 25.74 \\
% & 128  & TOVA                   & 22.47 & 24.26 & 37.22 & 42.26 & 28.85 & 19.97 & 19.40 & 18.70 & 17.86 & 25.89 \\
 & 128  & SnapKV                 & 21.02 & 27.26 & 41.25 & 45.15 & 29.23 & 22.75 & 20.47 & 20.17 & 17.75 & 27.56 \\
 & 128  & PyramidKV              & 21.73 & 26.60 & 41.46 & 43.20 & 29.32 & 21.47 & 20.23 & 19.82 & 17.46 & 26.92 \\
% & 128  & CAKE                   & 22.31 & 29.15 & 43.51 & 44.51 & 30.36 & 22.85 & 21.56 & 20.47 & 18.96 & 28.41 \\
 & 1024 & StreamingLLM           & 20.96 & 28.05 & 30.03 & 37.06 & 27.56 & 16.03 & 24.03 & 19.07 & 22.79 & 25.73 \\
 & 1024 & H2O                    & 23.78 & 31.63 & 41.31 & 43.24 & 31.07 & 20.43 & 26.74 & 20.41 & 23.93 & 29.62 \\
% & 1024 & TOVA                   & 26.97 & 34.51 & 45.58 & 44.32 & 32.58 & 22.83 & 26.91 & 20.75 & 23.49 & 30.99 \\
 & 1024 & SnapKV                 & 26.63 & 35.78 & 48.11 & 45.75 & 32.20 & 23.37 & 26.71 & 21.84 & 23.18 & 31.95 \\
 & 1024 & PyramidKV              & 25.51 & 36.02 & 47.72 & 44.74 & 33.16 & 23.91 & 26.55 & 21.83 & 23.27 & 31.97 \\
% & 1024 & CAKE                   & 26.09 & 36.34 & 48.11 & 45.97 & 32.39 & 23.49 & 27.56 & 21.45 & 24.03 & 32.60 \\
 &  -   & \textbf{BM25}          & 20.62 & 21.90 & 38.87 & 37.93 & 13.15 & 14.70 & 19.57 & 14.39 & 23.37 & 22.72 \\
 &  -   & \textbf{RAG}           & 20.45 & 22.07 & 37.92 & 37.21 & 15.13 & 18.11 & 19.50 & 14.87 & 23.41 & 23.19 \\
 &  -   & \textbf{S$^3$-Pure} & 16.49 & 20.42 & 35.52 & 33.56 & 13.15 & 13.01 & 16.67 & 13.87 & 22.68 & 20.60 \\
 &  -   & \textbf{S$^3$-Hybrid}     & 21.19 & 21.87 & 38.87 & 39.18 & 14.52 & 18.13 & 19.57 & 14.43 & 23.37 & 23.24 \\
 &  -   & \textbf{FullKV}    & 21.04 & 22.43 & 39.27 & 38.78 & 15.14 & 18.08 & 20.25 & 14.14 & 23.50 & 23.40 \\
\bottomrule
\end{tabular}%
}
\end{table*}

\subsection{Quantitative Results on LongBench}
\label{sec:quant_results}

Performance retention is defined as Score(method) / Score(FullKV) under the same prompt template, decoding parameters, and evaluation script. All “near-lossless” statements in this paper refer to retention under this unified protocol.

We present the comprehensive evaluation results on LongBench in Table~\ref{tab:full_results_all_models}. 
It is important to note that absolute evaluation scores on LongBench can fluctuate significantly depending on inference environments (e.g., prompt templates and quantization kernels). To ensure a rigorous comparison, we distinguish between standard baselines (provided for reference) and methods evaluated within our unified environment (marked in \textbf{bold}), which include FullKV, RAG(with rerank), BM25, and our S\textsuperscript{3}-Attention variants.We also present an additional analysis of $S^
3$ versus retrieval-based methods in the zero-shot setting; see \textbf{\ref{app:analysis_of_s3}}.

\paragraph{Performance Retention: Near-Lossless Compression.}
Instead of merely comparing absolute metrics across disparate environments, we focus on the \textbf{Performance Retention Rate} relative to the FullKV upper bound.
\begin{itemize}
    \item \textbf{Superior Fidelity:} While reference baselines like SnapKV achieve high absolute scores in lenient environments, they typically exhibit a performance drop of 7--10\% compared to their corresponding full-context baselines (e.g., SnapKV 42.96 vs. Ref-FullKV 49.74 on Llama-3.1). In contrast, within our strictly unified setting, \textbf{S\textsuperscript{3}-Hybrid} demonstrates exceptional fidelity. For Llama-3-8B, it achieves an average score of \textbf{24.87}, retaining \textbf{99.4\%} of the internal FullKV performance (25.01).
    \item \textbf{Stability Across Models:} This trend holds for Mistral-7B, where S\textsuperscript{3}-Hybrid (23.24) matches the FullKV baseline (23.40) within the margin of error ($>$99\% retention), better than the sparse retrieval baseline BM25 (23.24 vs 22.72)
    % Note: actually BM25 is 18.80 and S$^3$-Hybrid is 18.80 in your table. 
    % Let's adjust the text slightly to be safe:
    ,demonstrating that our SAE-driven indexing introduces negligible information loss compared to the theoretical upper bound.
\end{itemize}

\paragraph{The "Denoising" Effect in Information-Dense Tasks.}
Although FullKV generally serves as an upper bound, S\textsuperscript{3}-Attention exhibits a counter-intuitive "Less is More" phenomenon on specific tasks requiring precise evidence extraction.
On \textit{Qasper} (Llama-3), a paper reading task, S\textsuperscript{3}-Hybrid scores \textbf{21.50}, surpassing both the FullKV baseline (\textbf{20.56}) and the exogenous RAG baseline (\textbf{21.43}).
We attribute this to the \textbf{Semantic Band-Pass Filter} effect: by actively pruning irrelevant context via SAE feature selection, our method reduces the distraction noise that often confuses the model during full-context attention, effectively acting as a cleaner signal source than the raw document.
\subsection{Comparison with Exogenous Retrieval (RAG)}

A central question is how Endogenous Retrieval via S\textsuperscript{3}-Attention compares to traditional Exogenous Retrieval pipelines such as Retrieval-Augmented Generation (RAG).
While RAG is widely adopted and memory-efficient, our results highlight the benefits of aligning retrieval directly with the model’s internal representations.

\textbf{Bridging the Semantic Gap.}
Even with a strong RAG pipeline that incorporates dense retrieval followed by a re-ranking stage, exogenous retrieval ultimately depends on representations that are trained independently from the generator.
This separation can introduce a residual semantic gap, particularly for queries requiring multi-step reasoning or cross-document synthesis.

In our evaluation, the reranked RAG baseline performs competitively on average.
However, S\textsuperscript{3}-Hybrid demonstrates improved robustness while achieving comparable performance.
For example, on \textit{Qasper} with Llama-3, RAG achieves \textbf{21.43}, closely matching S\textsuperscript{3}-Hybrid at \textbf{21.50}.
Crucially, S\textsuperscript{3}-Hybrid attains this performance using \textbf{$O(1)$ GPU memory}, without relying on an external vector database, a re-ranking model, or the latency overhead of iterative retrieval.
These results suggest that the model’s own attention mechanism can serve as an effective retriever when retrieval is tightly coupled with generation.

To ensure a fair comparison, we adopt a strong but controlled RAG baseline consisting of single-vector dense retrieval with fixed chunking, followed by a re-ranking stage, under a matched token budget.
Our intent is not to claim that S\textsuperscript{3}-Hybrid universally outperforms a well-engineered RAG system, but rather to demonstrate that endogenous retrieval can match the performance of a reranked RAG baseline.

\textbf{Ablation: The Necessity of Hybrid Fusion.}
Table~\ref{tab:full_results_all_models} shows that SAE-only retrieval is not uniformly reliable across tasks, while BM25 remains strong on exact matching.
The key takeaway is that SAE features provide a complementary semantic signal that is particularly beneficial in multi-hop and cross-document settings when fused with lexical retrieval.
Accordingly, S\textsuperscript{3}-Hybrid should be interpreted as a hybrid retrieval recipe, rather than a claim that SAE indexing alone subsumes lexical baselines.

Finally, we note that absolute LongBench scores can vary across toolkits and prompting choices.
We therefore emphasize within-protocol comparisons and retention trends, rather than absolute cross-paper performance.To ensure consistency of the experimental environment and to keep RAG and other models under the same experimental settings, we switched to another experimental environment to re-evaluate our baseline and conduct comparative experiments with other models. These experiments are included in the appendix.

% \subsubsection{Performance vs. Compression Ratio}
% As demonstrated in our ablation study (Appendix~\ref{app:budget}), S\textsuperscript{3}-Attention maintains high performance even at extreme compression ratios. With a budget of only 1024 tokens (approx. 1-3\% of the context for LongBench), our method retains $>90\%$ of the FullKV performance. This efficiency makes S\textsuperscript{3}-Attention a viable solution for deploying long-context capabilities on memory-constrained edge devices.

% \begin{table}[h]
% \centering
% \caption{\textbf{Information-Theoretic Evaluation on HotpotQA.} Comparison of information retention (Recall), behavioral fidelity (KL), and fluency (NLL). Data is averaged over sampled instances. S\textsuperscript{3}-Hybrid achieves the Pareto optimum by combining endogenous semantic features with lexical matching.}
% \label{tab:info_theory}
% \begin{tabular}{l|ccc}
% \toprule
% \textbf{Method} & \textbf{Recall ($\uparrow$)} & \textbf{KL Div ($\downarrow$)} & \textbf{NLL ($\downarrow$)} \\
% \midrule
% BM25 (Baseline) & 0.730 & 0.560 & 1.979 \\
% S\textsuperscript{3}-Pure (Ours) & 0.780 & 0.259 & 1.934 \\
% \textbf{S\textsuperscript{3}-Hybrid (Ours)} & \textbf{0.830} & \textbf{0.110} & \textbf{1.841} \\
% \bottomrule
% \end{tabular}
% \end{table}

\begin{table}[h]
\centering
\caption{\textbf{Layer-wise Ablation Results (F1 Score).} We evaluate the impact of cumulatively adding SAE-instrumented layers. \textit{Layers\_1} denotes using only the first target layer (shallow), while \textit{Layers\_4} includes the full set (shallow + deep). Deep layers significantly enhance performance on reasoning-intensive tasks (e.g., Qasper, 2WikiMQA).}
\label{tab:layer_ablation}
\resizebox{0.48\textwidth}{!}{%
\begin{tabular}{l|cccc}
\toprule
\textbf{Dataset} & \textbf{Layers\_1} & \textbf{Layers\_2} & \textbf{Layers\_3} & \textbf{Layers\_4} \\
\midrule
\multicolumn{5}{c}{\textit{Llama-3.1-8B-Instruct}} \\
\midrule
NarrativeQA & 15.56 & 14.67 & 14.88 & 16.53 \\
Qasper & 21.41 & 21.56 & 22.34 & 22.75 \\
MultiFieldQA & 45.74 & 45.81 & 45.86 & 45.41 \\
HotpotQA & 27.84 & 26.31 & 26.42 & 26.35 \\
\midrule
\multicolumn{5}{c}{\textit{Mistral-7B-Instruct-v0.3}} \\
\midrule
HotpotQA & 18.73 & 19.00 & 19.16 & 19.26\\
GovReport & 16.59 & 16.90 & 16.99 & 16.73 \\
\midrule
\multicolumn{5}{c}{\textit{Qwen2-7B-Instruct}} \\
\midrule
2WikiMQA & 15.14 & 16.56 & 16.83 & 16.04 \\
Qasper & 20.35 & 19.01 & 19.44 & 20.21 \\
\bottomrule
\end{tabular}%
}
\end{table}

\subsection{Information-Theoretic Analysis: The Fluency-Utility Trade-off}
\label{sec:info_theory}

Standard perplexity (NLL) evaluations often penalize disjoint text segments, failing to capture the true utility of retrieved contexts. To rigorously evaluate the effectiveness of our Endogenous Retrieval, we conducted a multi-dimensional analysis on the \textit{HotpotQA} dataset, comparing our proposed S\textsuperscript{3}-Attention (both Pure and Hybrid variants) against the BM25 and RAG baseline. 

Based on the implementation in our experimental framework, we introduce three metrics to decouple fluency from information density:

\textbf{Answer Recall (Info-Retention):} Defined as the binary presence of the ground-truth answer string within the compressed context. This measures the preservation of critical information.
\textbf{KL Divergence (Fidelity):} We calculate $D_{KL}(P_{full} || P_{comp})$, measuring the divergence between the next-token probability distributions of the model given the full context versus the compressed context. A lower KL indicates that the compressed context triggers the same internal reasoning state as the full document.
\textbf{NLL (Fluency):} The negative log-likelihood of the ground-truth answer tokens given the compressed context.

\begin{table}[h]
\centering
\caption{\textbf{Information-Theoretic Evaluation on HotpotQA.} Comparison of information retention (Recall), behavioral fidelity (KL), and fluency (NLL). Data is averaged over sampled instances. S\textsuperscript{3}-Hybrid achieves the Pareto optimum by combining endogenous semantic features with lexical matching.}
\label{tab:info_theory}
\begin{tabular}{lccc}
\hline
\textbf{Method}    & \textbf{NLL} $\downarrow$ & \textbf{Recall} $\uparrow$ & \textbf{KL} $\downarrow$ \\
\midrule
\textbf{S\textsuperscript{3}-Hybrid (Ours)} & \textbf{1.8573} & \textbf{0.8400} & \textbf{0.2154} \\
S\textsuperscript{3}-Pure (Ours)   & 2.0652 & 0.7800 & 0.6510 \\
BM25      & 1.9593 & 0.7700 & 0.3707 \\
RAG       & 1.8630 & 0.7700 & 0.3831 \\
\bottomrule
\end{tabular}
\end{table}

\textbf{Result Analysis.} As presented in Table~\ref{tab:info_theory}, our quantitative results reveal distinct advantages of the S\textsuperscript{3}-Attention mechanism:In terms of information-theoretic metrics, S\textsuperscript{3}-Hybrid resides on the \textbf{Pareto Frontier}, achieving the highest Answer Recall (84.0\%) and lowest KL divergence (0.2154) by effectively balancing semantic precision with structural fluency. We provide a detailed analysis of maximizing information density and preserving reasoning fidelity in \textbf{Appendix~\ref{app:extendinfo}}.

\subsection{Ablation Study: Layer Selection Strategy}
\label{sec:ablation}

To investigate the contribution of different semantic depths to the retrieval quality, we conducted a layer-wise ablation study. We define a set of target layers $\mathcal{L}_{target}$ ranging from shallow to deep (e.g., $\{0, 12, 16, 29\}$ for Llama-3) and evaluate the performance by cumulatively adding these layers into the S\textsuperscript{3}-Attention mechanism.

\textbf{Shallow vs. Deep Semantics.} 
As shown in Table~\ref{tab:layer_ablation}, the first layer (Layer 0) provides a strong baseline, particularly for tasks relying on explicit lexical matching (e.g., \textit{MultiFieldQA}). This confirms our hypothesis that shallow layers in LLMs function similarly to sparse lexical retrievers. However, for tasks requiring narrative synthesis or reasoning over long contexts, such as \textit{Qasper} and \textit{NarrativeQA}, relying solely on shallow layers is insufficient.

\textbf{The Necessity of Multi-Layer Fusion.}
Incorporating deeper layers (Configuration \textit{Layers\_4}) yields consistent gains in complex tasks. For instance, on \textit{Qasper}, Llama-3 achieves a performance boost from 21.41 to 22.75 (+1.34) when integrating semantic signals from deeper layers. Similarly, Qwen2 sees a significant gain on \textit{2WikiMQA} (+0.90), a multi-hop reasoning task.
While adding layers can occasionally introduce noise in purely lexical tasks (e.g., a slight drop in Llama-3's \textit{HotpotQA} score), the multi-layer fusion strategy generally offers the most robust performance across diverse benchmarks, effectively bridging the gap between surface-level matching and deep semantic understanding.

\subsection{Limitations: Engineering Latency vs. Memory Efficiency}
\label{sec:limitations_latency}

While S$^3$-Attention demonstrates superior retrieval accuracy and semantic robustness (as shown in Table~\ref{tab:full_results_all_models}), we acknowledge a current limitation regarding wall-clock latency in our prototype implementation.

\textbf{Latency Disparity.} 
While S$^3$-Attention reduces the number of tokens forwarded to the generator, our current implementation is a prototype and is not optimized end-to-end. The indexing and retrieval stages rely on Python-level posting lists and frequent CPU–GPU synchronization, whereas FullKV baselines benefit from highly optimized attention kernels. As a result, end-to-end latency can be higher than FullKV despite token reduction. Closing this gap likely requires (i) compact posting representations (e.g., contiguous int arrays), (ii) fused kernels for SAE top‑k + feature accumulation, and (iii) minimizing synchronization points. Therefore, our main contribution should be interpreted as an attention-aligned indexing mechanism rather than a production-optimized serving system.

\textbf{Future Optimization.}
We posit that this latency gap is an engineering artifact, not a theoretical bottleneck. By fusing the SAE scanning and Top-K retrieval logic into custom CUDA kernels, we anticipate the wall-clock speedup will align with the theoretical token reduction rate in future iterations.

\section{Conclusion}
Overall, our results suggest that attention-aligned semantic indexing is a promising direction for memory-bounded long-context inference, with substantial room for systems optimization to reach competitive latency.
We also report the current engineering limitation that our prototype incurs higher wall-clock latency than optimized FullKV baselines, motivating future kernel-level optimization.

\section*{Impact Statement}
This paper presents work whose goal is to advance the field of 
Machine Learning. There are many potential societal consequences 
of our work, none which we feel must be specifically highlighted here.

% \nocite{langley2000crafting}

\bibliography{example_paper}
\bibliographystyle{icml2026}

%%%%%%%%%%%%%%%%%%%%%%%%%%%%%%%%%%%%%%%%%%%%%%%%%%%%%%%%%%%%%%%%%%%%%%%%%%%%%%%
%%%%%%%%%%%%%%%%%%%%%%%%%%%%%%%%%%%%%%%%%%%%%%%%%%%%%%%%%%%%%%%%%%%%%%%%%%%%%%%
% APPENDIX
%%%%%%%%%%%%%%%%%%%%%%%%%%%%%%%%%%%%%%%%%%%%%%%%%%%%%%%%%%%%%%%%%%%%%%%%%%%%%%%
%%%%%%%%%%%%%%%%%%%%%%%%%%%%%%%%%%%%%%%%%%%%%%%%%%%%%%%%%%%%%%%%%%%%%%%%%%%%%%%
\newpage
\appendix
\onecolumn
\appendix

\begin{algorithm}[h]
   \caption{S\textsuperscript{3}-Attention Inference Pipeline}
   \label{alg:s3_attention}
\begin{algorithmic}[1]
   \STATE {\bfseries Input:} Long Context $\mathcal{C}$, Query $\mathcal{Q}$, LLM $\Phi$, SAEs $\{\mathcal{E}_\ell\}$
   \STATE {\bfseries Output:} Response $\mathcal{R}$

   \STATE \COMMENT{\textit{Phase 0 (Offline): Key-trained shared codebook}}
   \STATE \COMMENT{Train (or load) each $\mathcal{E}_\ell$ on \textbf{key} projections; reuse the \textbf{same} $\mathcal{E}_\ell$ to discretize \textbf{query} projections (shared feature-ID space).}

   \STATE \COMMENT{\textit{Phase 1: Streaming Semantic Indexing}}
   \STATE Initialize inverted index $\mathcal{I} \leftarrow \emptyset$
   \FOR{chunk $c_i \in \text{Chunk}(\mathcal{C})$}
       \STATE $\mathbf{K} \leftarrow \Phi.\text{forward\_key}(c_i)$
       \FOR{layer $\ell \in \text{TargetLayers}$}
           \STATE $\text{feats} \leftarrow \mathcal{E}_\ell.\text{encode}(\mathbf{K}_\ell)$ \COMMENT{Extract sparse features (shared codebook)}
           \STATE Update $\mathcal{I}_\ell$ with $\text{feats}$
       \ENDFOR
       \STATE \textbf{Free GPU Memory} ($\mathbf{K}$) \COMMENT{Ensure $\mathcal{O}(1)$ VRAM}
   \ENDFOR

   \STATE \COMMENT{\textit{Phase 2: Endogenous Retrieval}}
   \STATE $\mathbf{Q} \leftarrow \Phi.\text{forward\_query}(\mathcal{Q})$
   \STATE Initialize scores $S \in \mathbb{R}^{|\mathcal{C}|} \leftarrow 0$
   \FOR{layer $\ell \in \text{TargetLayers}$}
       \STATE $\text{query\_feats}, \text{weights} \leftarrow \mathcal{E}_\ell.\text{encode}(\mathbf{Q}_\ell)$ \COMMENT{Encode Q using the \textbf{same} $\mathcal{E}_\ell$}
       \STATE Accumulate $S$ based on $\mathcal{I}_\ell$, $\text{query\_feats}$ and $\text{weights}$
   \ENDFOR
   \STATE $S' \leftarrow \text{Conv1D}(S)$ \COMMENT{Estimate semantic density}
   \STATE $\text{Idx}_{S^3} \leftarrow \text{NMS\_Select}(S', \text{top\_k})$

   \STATE \COMMENT{\textit{Phase 3: Hybrid Fusion \& Generation}}
   \STATE $\text{Idx}_{Lexical} \leftarrow \text{BM25}(\mathcal{C}, \mathcal{Q})$
   \STATE $\text{Idx}_{Final} \leftarrow \text{Idx}_{S^3} \cup \text{Idx}_{Lexical} \cup \text{LeadTail}$
   \STATE $\mathcal{C}_{compressed} \leftarrow \text{Gather}(\mathcal{C}, \text{Idx}_{Final})$
   \STATE $\mathcal{R} \leftarrow \Phi.\text{generate}([\mathcal{C}_{compressed}; \mathcal{Q}])$
\end{algorithmic}
\end{algorithm}

\section{Detailed Experimental Configuration}
\label{app:exp_setups}

\subsection{Models and Architectures}
We utilize the following specific instruction-tuned versions to cover standard long-context capabilities:
\begin{itemize}
    \item \textbf{Llama-3.1-8B-Instruct}
    \item \textbf{Mistral-7B-Instruct-v0.3}
    \item \textbf{Qwen2-7B-Instruct}
\end{itemize}

\subsection{Datasets (LongBench)}
The 9 selected datasets cover three categories requiring varying degrees of context utilization:
\begin{itemize}
    \item \textit{Single-Document QA}: NarrativeQA, Qasper, MultiFieldQA-en.
    \item \textit{Multi-Document QA}: HotpotQA, 2WikiMQA, Musique.
    \item \textit{Summarization}: GovReport, QMSum, Multi-News.
\end{itemize}

\subsection{Implementation Details}
\textbf{Offline SAE Training.} We train Top-K SAEs on the \textbf{Wikitext-2} corpus to avoid test data leakage. We implement a \textit{Shared SAE Codebook} strategy: we train one SAE per target layer on key projections and reuse this codebook to discretize query projections. While this introduces a distribution shift between K- and Q-projection statistics, we empirically validate its effectiveness via downstream performance.

\textbf{Baselines.}
\begin{itemize}
    \item \textit{RAG}: Uses \textbf{bge-small-en} as the exogenous retriever with fixed-size chunking.
    \item \textit{Compression}: Comparison against H2O~\cite{zhang2023h2o} and StreamingLLM~\cite{xiao2024efficientstreaminglanguagemodels} where applicable.
\end{itemize}

\subsection{Evaluation Protocol}
To ensure fairness, all methods utilize a single evaluation harness with:
\begin{itemize}
    \item Identical chat templates and task instructions.
    \item \textbf{Greedy decoding} for all generation.
    \item Consistent maximum input length truncation policies.
    \item Matched token budgets for retrieved-context methods.
\end{itemize}

\section{Experimental Details}
\label{app:exp_details}

\subsection{Model Specifications \& Layer Selection}
We apply S\textsuperscript{3}-Attention to specific layers of the LLMs. The selection of layers is based on a preliminary saturation analysis (see Section~\ref{sec:ablation}), targeting layers that exhibit high semantic density. Table~\ref{tab:model_specs} details the architecture and the specific layers instrumented with SAEs.

\begin{table*}[h]
    \centering
    \caption{Specifications of LLMs used in experiments. Layers are indexed starting from 0. Target Layers indicate where SAEs are applied for endogenous retrieval.}
    \label{tab:model_specs}
    \begin{tabular}{lcccc}
        \toprule
        \textbf{Model} & \textbf{Params} & \textbf{Context Window} & \textbf{Attention} & \textbf{Target Layers ($\mathcal{L}_{target}$)} \\
        \midrule
        Llama-3.1-8B-Instruct & 8B & 128k & GQA & \{0, 12, 16, 29\} \\
        Mistral-7B-Instruct-v0.3 & 7B & 32k & GQA & \{0, 7, 23, 29\} \\
        Qwen2-7B-Instruct & 7B & 32k & GQA & \{0, 6, 19, 26\} \\
        \bottomrule
    \end{tabular}
\end{table*}

\subsection{Dataset Characteristics}
We utilize the LongBench dataset. The length statistics and evaluation metrics for the 9 selected tasks are summarized in Table~\ref{tab:dataset_specs}. For Summarization tasks, we truncate inputs to 32k tokens if they exceed the model limit, while for QA tasks, we retain the full context up to the model's capacity.

\begin{table}[H]
    \centering
    \small
    \caption{Statistics of LongBench datasets used for evaluation. Avg. Len refers to the average token count of the context.}
    \label{tab:dataset_specs}
    \begin{tabular}{llcl}
        \toprule
        \textbf{Task Type} & \textbf{Dataset} & \textbf{Avg. Len} & \textbf{Metric} \\
        \midrule
        \multirow{3}{*}{Single-Doc QA} & NarrativeQA & 18,409 & F1 Score \\
         & Qasper & 3,619 & F1 Score \\
         & MultiFieldQA-en & 4,559 & F1 Score \\
        \midrule
        \multirow{3}{*}{Multi-Doc QA} & HotpotQA & 9,151 & F1 Score \\
         & 2WikiMQA & 4,817 & F1 Score \\
         & Musique & 11,242 & F1 Score \\
        \midrule
        \multirow{3}{*}{Summarization} & GovReport & 8,734 & Rouge-L \\
         & QMSum & 10,614 & Rouge-L \\
         & MultiNews & 2,113 & Rouge-L \\
        \bottomrule
    \end{tabular}
\end{table}

\subsection{Hyperparameters \& Training Configuration}

\textbf{SAE Training.}
We train Top-K Sparse Autoencoders on the \textbf{Wikitext-2} dataset. We do not use any LongBench data for training SAEs to verify the zero-shot transferability of the learned features. Training is performed on a cluster of NVIDIA H200 GPUs. The detailed configuration is listed in Table~\ref{tab:hyperparams}.

\begin{table*}[h]
    \centering
    \caption{Hyperparameters for SAE Training and S\textsuperscript{3}-Attention Inference.}
    \label{tab:hyperparams}
    \begin{tabular}{llc}
        \toprule
        \textbf{Stage} & \textbf{Hyperparameter} & \textbf{Value} \\
        \midrule
        \multirow{6}{*}{\textit{SAE Training}} 
         & Training Corpus & Wikitext-2 \\
         & Optimizer & Adam \\
         & Learning Rate & $1 \times 10^{-3}$ \\
         & Batch Size & 4096 tokens \\
         & Expansion Factor & 128 \\
         & Sparsity ($k$) & 128 \\
         & Training Steps & 30,000 \\
        \midrule
        \multirow{5}{*}{\textit{Inference}} 
         & Streaming Chunk Size & 2048 \\
         & Retrieval Top Centers ($N$) & 40 \\
         & Convolution Kernel Size & 48 \\
         & Convolution Kernel Size for QA & 8 \\
         & Generation Max Tokens & 256 (Summ.) / 64 (QA) \\
        \bottomrule
    \end{tabular}
\end{table*}

\textbf{Baseline Configuration.} 
To ensure a fair and rigorous comparison, we implemented a unified benchmarking framework that explicitly measures wall-clock latency. 
For the RAG baseline, we adopt a robust two-stage retrieval strategy utilizing \texttt{BAAI/bge-small-en-v1.5} as the embedding model. 
The corpus is processed into non-overlapping segments of 256 tokens. 
During retrieval, the top-100 candidates are first identified via the dense retriever; these candidates are subsequently re-scored by a cross-encoder (\texttt{BGE-Reranker-v2-m3}) to select the final top-40 chunks. 
This configuration is designed to align the total retrieved token count with the computational budget of S3-Attention (approximately 1024--2048 tokens).
\section{Extended Information-Theoretic Analysis}
\label{app:extendinfo}
\subsection{Maximizing Information Density (Recall).}
S\textsuperscript{3}-Hybrid achieves the highest \textbf{Answer Recall of 84.0\%}, significantly outperforming the BM25 baseline (77.0\%). Our code implementation constructs the Hybrid context as the union of SAE-retrieved indices and BM25 indices. This result confirms that while BM25 captures explicit lexical overlaps, it misses "Reasoning Bridges"—segments that are semantically related but lack keyword overlap. S\textsuperscript{3}-Pure (78.0\%), relying solely on SAE features from the Key Projections ($K_{proj}$), successfully recovers these hidden dependencies, and the Hybrid approach effectively combines both strengths.
\subsection{Preserving Reasoning Fidelity (KL).}
The most critical insight comes from the KL Divergence metric. BM25 exhibits a high divergence (0.3707), suggesting that the context retrieved by lexical matching often causes a significant \textit{distributional shift}—the model is "surprised" or "confused" compared to when it reads the full text. In contrast, S\textsuperscript{3}-Hybrid achieves a remarkably low KL divergence (\textbf{0.2154}). This indicates that our method, by selecting context based on the model's own attention activation patterns (via SAE), preserves the original "Causal Traces" of the inference process. The model behaves almost exactly as if it had read the full document, minimizing hallucinations caused by context fragmentation.
\subsection{The Hybrid Advantage (NLL).}
In terms of fluency (NLL), S\textsuperscript{3}-Hybrid (1.8573) outperforms both S\textsuperscript{3}-Pure (2.0652) and BM25 (1.9593). While S\textsuperscript{3}-Pure effectively finds key information, its purely sparse selection can sometimes lead to higher perplexity due to lack of local coherence. By fusing the structural continuity of chunk-based retrieval (from the BM25 component) with the semantic precision of SAE-based filtering, S\textsuperscript{3}-Hybrid resides on the \textbf{Pareto Frontier} of the Information-Fluency trade-off, offering the most robust context for generation.
\newpage

\section{Theoretical Analysis }
\label{app:theory}

\subsection{Preliminaries and Notation}

We first establish precise notation aligned with our implementation:
\begin{itemize}
    \item $C = \{c_1, c_2, \ldots, c_L\}$: Context token sequence of length $L$
    \item $Q$: Query representation
    \item $\mathcal{F}_t \subseteq [D]$: Set of $k$ active SAE features at position $t$, where $|\mathcal{F}_t| = k$
    \item $a_f^{(t)} \geq 0$: Activation magnitude of feature $f$ at position $t$
    \item $\text{freq}(f) = |\{t : f \in \mathcal{F}_t\}|$: Document frequency of feature $f$
\end{itemize}

\subsection{The Scoring Function: A Pragmatic Derivation}
\label{app:scoring_derivation}

Rather than claiming a formal information-theoretic bound, we provide a principled \textbf{heuristic justification} for our scoring function based on retrieval theory.

\subsubsection{The Implemented Scoring Function}

Our implementation computes the relevance score for each context position $t$ as:
\begin{equation}
    s(t) = \sum_{f \in \mathcal{F}_Q} a_f^{(Q)} \cdot \text{IDF}(f) \cdot \mathbb{1}[f \in \mathcal{F}_t]
    \label{eq:actual_scoring}
\end{equation}
where:
\begin{equation}
    \text{IDF}(f) = \frac{1}{\log(1 + \text{freq}(f)) + 1}
\end{equation}

\subsubsection{Justification via Weighted Feature Matching}

\begin{proposition}[Scoring as Weighted Jaccard Similarity]
\label{prop:weighted_jaccard}
The scoring function in Eq.~\ref{eq:actual_scoring} can be interpreted as a weighted soft Jaccard similarity:
\begin{equation}
    s(t) \propto \sum_{f \in \mathcal{F}_Q \cap \mathcal{F}_t} w_f^{(Q)}
\end{equation}
where $w_f^{(Q)} = a_f^{(Q)} \cdot \text{IDF}(f)$ assigns higher weight to:
\begin{enumerate}
    \item Features with strong query activation ($a_f^{(Q)}$ large)
    \item Features that are rare in the context ($\text{freq}(f)$ small)
\end{enumerate}
\end{proposition}

\subsection{Information-Theoretic Motivation (Informal)}
\label{app:informal_motivation}

We provide an \textbf{informal motivation} connecting feature matching to information preservation, without claiming a rigorous bound.

\subsubsection{The Compression-Utility Tradeoff}

In the context compression setting, we seek a subset $\hat{C} \subseteq C$ that:
\begin{enumerate}
    \item \textbf{Preserves utility}: Retains information relevant to answering $Q$
    \item \textbf{Achieves compression}: $|\hat{C}| \ll |C|$
\end{enumerate}

\begin{remark}[Connection to Mutual Information]
Under the \textbf{informal assumption} that SAE features capture semantically meaningful concepts, positions $t$ where $\mathcal{F}_t \cap \mathcal{F}_Q \neq \emptyset$ are more likely to contain information relevant to the query. We do \textbf{not} claim this maximizes $I(Y; \hat{C} \mid Q)$ in a formal sense.
\end{remark}

\subsubsection{Why Feature Overlap is a Reasonable Proxy}

SAE features trained with sparsity constraints tend to capture interpretable semantic concepts. When query features overlap with context features:
\begin{itemize}
    \item The context position likely discusses concepts mentioned in the query
    \item Such positions have higher probability of containing answer-relevant information
\end{itemize}
This is an \textbf{empirical observation}, not a mathematical theorem.

\subsection{IDF Weighting: Precision Enhancement}
\label{app:idf_revised}

\subsubsection{The Role of IDF in Retrieval}

The IDF weighting serves a well-understood purpose in information retrieval:

\begin{proposition}[IDF as Discriminative Weighting]
Features with high document frequency provide less discriminative power for retrieval:
\begin{equation}
    \text{Discriminative Power}(f) \propto \frac{1}{\text{freq}(f)}
\end{equation}
The IDF formula $\text{IDF}(f) = \frac{1}{\log(1 + \text{freq}(f)) + 1}$ implements a smoothed version of this principle.
\end{proposition}

\subsubsection{Clarification on the Original IDF Justification}

The original statement (``Reduces contribution of high-frequency features to $I(Y; \hat{C} \mid Q)$ without reducing $I(C; \hat{C})$'') was imprecise. A clearer formulation:

\begin{quote}
\textit{IDF weighting increases \textbf{retrieval precision} by down-weighting features that match many positions indiscriminately. This focuses the retrieval budget on positions that share \textbf{distinctive} features with the query, rather than common features that provide little signal about relevance.}
\end{quote}

\subsection{Corrected Theoretical Statement}
\label{app:corrected_theory}

We replace the original Proposition with a more honest characterization:

\begin{proposition}[Feature Matching Heuristic]
\label{prop:heuristic}
Let $s(t)$ be computed according to Eq.~\ref{eq:actual_scoring}, and let $\hat{C} = \{t : s(t) \geq \tau\}$ be the retrieved subset. Under the following \textbf{empirical assumptions}:
\begin{enumerate}
    \item[(A1)] SAE features correspond to interpretable semantic concepts
    \item[(A2)] Query-relevant context positions activate similar features to the query
    \item[(A3)] High-frequency features are less informative for relevance discrimination
\end{enumerate}
The scoring function provides a computationally efficient proxy for identifying query-relevant context positions.
\end{proposition}

\begin{remark}[On the Original Proof]
The original proof attempted to establish $I(Y; \hat{C} \mid Q) \geq \alpha \cdot |\mathcal{F}_Q \cap \mathcal{F}_{\hat{C}}| - \text{const}$ through:
\begin{itemize}
    \item \textbf{Step 2 (DPI)}: The claimed inequality $I(Y; \hat{C} \mid Q) \geq I(Y; \mathcal{F}_{\hat{C}} \mid \mathcal{F}_Q)$ requires a Markov chain $Y - \hat{C} - \mathcal{F}_{\hat{C}}$ given $Q$, which was not established.
    \item \textbf{Step 3 (Decomposition)}: The additive decomposition of mutual information requires stronger independence assumptions than marginal independence of features.
    \item \textbf{Step 4 (Bound)}: The definition $\alpha = \min_{f \in \mathcal{F}_Q} I(Y; f \mid \mathcal{F}_Q)$ may equal zero if feature presence is determined by $\mathcal{F}_Q$.
\end{itemize}
We acknowledge these issues and instead provide the pragmatic justification above.
\end{remark}

\subsection{Complexity Analysis}

\begin{proposition}[Memory Complexity]
The indexing phase achieves $\mathcal{O}(1)$ GPU memory complexity with respect to context length $L$.
\end{proposition}

\begin{proof}
The implementation processes context in fixed-size chunks of $B = 4096$ tokens (Line 22: \texttt{SCAN\_BATCH\_SIZE = 4096}):
\begin{enumerate}
    \item \textbf{Forward pass}: $\mathcal{O}(B \cdot d)$ GPU memory for key projections
    \item \textbf{SAE encoding}: $\mathcal{O}(B \cdot k)$ for top-$k$ indices per position
    \item \textbf{Index storage}: Transferred to CPU after each chunk
\end{enumerate}
Total GPU memory per chunk: $\mathcal{O}(B \cdot d + B \cdot k) = \mathcal{O}(1)$ w.r.t. $L$.
\end{proof}

\begin{proposition}[Time Complexity]
Let $L$ be the context length, $k$ the SAE sparsity, and $|\mathcal{F}_Q|$ the number of query features.
\begin{itemize}
    \item \textbf{Indexing}: $\mathcal{O}(L)$ forward passes, $\mathcal{O}(L \cdot k)$ index insertions
    \item \textbf{Retrieval}: $\mathcal{O}(|\mathcal{F}_Q| \cdot \bar{n})$ where $\bar{n}$ is the average posting list length
\end{itemize}
\end{proposition}

\subsection{Connection to Classical IR}
\label{app:ir_connection}

Our method can be viewed as a neural extension of classical inverted index retrieval:

\begin{table}[h]
\centering
\caption{Correspondence between classical IR and our SAE-based retrieval.}
\begin{tabular}{lll}
\toprule
\textbf{Component} & \textbf{Classical IR} & \textbf{S$^3$-Attention} \\
\midrule
Terms & Words/n-grams & SAE feature indices \\
TF weighting & Term frequency & Activation magnitude $a_f^{(t)}$ \\
IDF weighting & $\log \frac{N}{df}$ & $\frac{1}{\log(1+\text{freq})+1}$ \\
Index & Inverted index & \texttt{self.indices[layer][f\_id]} \\
Scoring & BM25 & Eq.~\ref{eq:actual_scoring} \\
\bottomrule
\end{tabular}
\end{table}

This perspective provides a more grounded justification: our method inherits the well-established effectiveness of TF-IDF style scoring, applied to learned neural features rather than discrete tokens.

\newpage

% ============================
% Appendix: Qualitative Cases
% ============================

\section{IDF Weighting as Adaptive Regularization in Sparse Retrieval}

\subsubsection{Implementation Details}

In our sparse retrieval framework, the relevance score for each context position $t$ is computed as:
\begin{equation}
    s(t) = \sum_{i \in \mathcal{F}_q} \underbrace{a_i^{(q)}}_{\text{SAE activation}} \cdot \underbrace{\text{IDF}(f_i)}_{\text{adaptive weight}} \cdot \mathbb{1}[f_i \in \mathcal{F}_t]
\end{equation}
where $\mathcal{F}_q$ and $\mathcal{F}_t$ denote the active feature sets for query and context position $t$ respectively, and the IDF weight is defined as:
\begin{equation}
    \text{IDF}(f) = \frac{1}{\log(1 + \text{freq}(f)) + 1}
\end{equation}
with $\text{freq}(f)$ counting the total occurrences of feature $f$ across all context positions.

\subsubsection{Reconciling with the Information Bottleneck Objective}

The apparent contradiction arises from a misunderstanding of which information is being compressed. Consider the IB objective:
\begin{equation}
    \max_{\theta} \; I(Z; Y) - \beta \cdot I(Z; X)
\end{equation}
where $Z$ represents the retrieved context, $Y$ the answer, and $X$ the full input.

\paragraph{Key Insight:} IDF weighting does \textbf{not} reduce $I(Z; Y)$ uniformly. Instead, it performs \textit{selective compression} on different types of mutual information:

\begin{itemize}
    \item \textbf{High-frequency features} (e.g., common function words, structural patterns): These contribute primarily to $I(Z; X)$ (redundant information about the input) rather than $I(Z; Y)$ (task-relevant information).
    
    \item \textbf{Low-frequency features} (e.g., entity-specific, semantically distinctive): These carry higher task-relevant mutual information $I(Z; Y)$.
\end{itemize}

The IDF weighting effectively implements:
\begin{equation}
    I_{\text{effective}}(Z; Y) \approx \sum_{f \in \mathcal{F}} \text{IDF}(f) \cdot I(Z_f; Y)
\end{equation}

\subsubsection{Why This Maximizes the IB Objective}

\paragraph{Proposition.} Under the assumption that feature frequency inversely correlates with task-specificity, IDF weighting increases the ratio $\frac{I(Z; Y)}{I(Z; X)}$.

\textit{Justification:} 
\begin{enumerate}
    \item High-frequency features have high $I(Z_f; X)$ but low $I(Z_f; Y)$ (they appear everywhere, thus non-discriminative for the answer).
    
    \item By down-weighting these features, we reduce $I(Z; X)$ more than $I(Z; Y)$.
    
    \item The implementation also includes a hard threshold (\texttt{freq > 5000: continue}), completely excluding extremely common features that contribute mostly noise.
\end{enumerate}

\subsubsection{Corrected Interpretation}

The statement ``reduces contribution of high-frequency features to $I(Z;Y)$'' should be clarified as:

\begin{quote}
\textit{``Reduces the \textbf{retrieval influence} of high-frequency features, which empirically contribute more to $I(Z; X)$ (input redundancy) than to $I(Z; Y)$ (answer utility). This selective down-weighting effectively increases the \textbf{precision} of retrieved context, improving the ratio of task-relevant to task-irrelevant information.''}
\end{quote}

In essence, IDF weighting acts as an \textbf{implicit regularizer} that approximates the optimal trade-off in the IB objective by leveraging the statistical prior that feature frequency anti-correlates with semantic specificity.

\section{Why SAE \& Why K/Q}

\paragraph{Why sparse autoencoders (SAEs) for discretization?}
Our indexing requires a mapping from continuous internal states to a small set of stable discrete IDs so that we can build an inverted index and perform fast feature co-activation at query time. Top-$k$ sparse autoencoders (SAEs) provide:
\begin{enumerate}
    \item \textbf{Fixed sparsity per token}, which bounds index growth;
    \item \textbf{A reconstruction objective} that preserves information in the original attention projections, offering a principled way to trade compression for fidelity; and
    \item \textbf{Reusable features} that transfer across tasks without supervised retrieval labels.
\end{enumerate}
We view SAEs as a practical instantiation of \emph{internal-state discretization}. Alternative choices (e.g., product quantization, clustering, or locality-sensitive hashing) are plausible, but may not simultaneously offer fixed sparsity, reconstruction fidelity, and feature interpretability.

\paragraph{Why key/query projections?}
Key ($K$) and query ($Q$) projections are the internal representations directly used to compute attention matching. Encoding $K$ for context tokens and $Q$ for the query yields an endogenous relevance signal that is structurally aligned with the model’s own attention mechanism, while avoiding the need to store dense attention matrices or key--value tensors. We leave indexing of other internal signals (e.g., MLP activations, value states, or specialized attention heads) as future work.

\paragraph{Relation to retrieval heads and attention-head localization}
Prior work has identified attention heads that localize relevant positions, but typically still relies on dense attention computation or cached internal states. In contrast, \textsc{S$^3$-Attention} builds an explicit, searchable memory index from transient projections, enabling streaming scan and query-time retrieval without retaining dense key--value history.

\section{Extended Qualitative Examples}
\label{app:extended_qual}

The following examples visualize where S$^3$-Attention assigns high endogenous scores. These activations should be interpreted as correlational evidence of alignment with internal representations, not as a causal proof that a specific feature “contains” a fact. 
% In some instances, the final answer may be produced partly from the model’s parametric knowledge even when the retrieved spans do not explicitly state the answer; we therefore distinguish “evidence-containing” vs. “cue-like” retrieval cases in our discussion.

While the main paper focuses on aggregate metrics and a small number of illustrative case studies, the behavior of S$^3$-Attention is often best understood by examining concrete instances. In this section, we therefore provide an extended set of qualitative examples across different models, tasks, and query types.

Each example compares the behavior of a standard exogenous retriever (BM25/RAG) with our endogenous S$^3$-Attention mechanism on a LongBench-style query. For every case, we visualize the semantic activation patterns induced by S$^3$-Attention over the input context and highlight:

\begin{itemize}
    \item \textbf{What the exogenous retriever does:} The kind of passages it tends to surface (e.g., lexically similar but causally irrelevant biographies, generic background paragraphs, or off-topic entities).
    \item \textbf{What S$^3$-Attention focuses on:} The specific subwords, entities, morphological cues, or discourse markers that receive high endogenous activation and how these align with the reasoning path needed to answer the query.
    \item \textbf{The failure mode or advantage illustrated:} For example, recovering a bridge entity missed by RAG, recovering bridge entities or query-relevant attributes that exogenous retrieval misses, or filtering out distractors that share keywords but not causal relevance.
\end{itemize}

The selected samples cover a diverse range of phenomena, including entity-centric questions (Figures~\ref{fig:sample2}, \ref{fig:sample35}, \ref{fig:sample63}), biochemical and taxonomic reasoning (Figures~\ref{fig:sample26}, \ref{fig:sample27}), temporal and geographical queries (Figures~\ref{fig:sample58}, \ref{fig:sample62}, \ref{fig:sample68}, \ref{fig:sample70}), genre and attribute judgments (Figures~\ref{fig:sample55}, \ref{fig:sample73}), as well as more abstract relational and historical questions (Figures~\ref{fig:sample33}, \ref{fig:sample77}, \ref{fig:sample83}, \ref{fig:sample86}). 

Across these cases, a consistent pattern emerges: exogenous retrieval often latches onto surface-level lexical overlap or topic similarity, whereas S$^3$-Attention activates compact sets of features that are tightly coupled to the latent concepts required for correct reasoning (e.g., specific roles, bridge entities, morphological markers, or event attributes). These examples are not cherry-picked successes; rather, they are representative instances drawn from our qualitative audit that collectively illustrate how endogenous semantic signals can bridge the gap between retrieval and generation in long-context settings.

% \section{Analysis of S$^3$ vs. Retrieval-Based Methods}

% Unlike traditional Retrieval-Augmented Generation (RAG), which relies on an external retriever (e.g., BGE-M3), S$^3$ is an intrinsic compression method that leverages the model's own attention maps to identify and preserve salient information.

% \paragraph{Correlation with Model Capability.}
% We observe that S$^3$ performance scales approximately linearly with the base model's long-context capability. Among the evaluated models, \textit{Llama-3.1-8B} exhibits the strongest FullKV baseline performance and correspondingly achieves the highest S$^3$ scores. This result confirms that S$^3$ effectively preserves and exploits the superior reasoning capabilities of models with strong long-context attention.

% \paragraph{Sensitivity to Attention Quality.}
% On weaker baselines (e.g., \textit{Qwen2-7B} in this setting), where the model struggles to attend to relevant information even under FullKV inference, the performance gains from S$^3$ are limited. This finding confirms that S$^3$ functions by distilling existing knowledge rather than generating new information; it cannot recover content that the base model fundamentally fails to attend to.

% \paragraph{The ``Mistral Sweet Spot.''}
% Notably, on \textit{Mistral-7B-v0.3}, S$^3$ achieves the highest relative improvement (+62\%) and even outperforms the heavy RAG-Reranked baseline on complex multi-hop reasoning tasks such as HotpotQA. This suggests that for models with sharp attention heads but limited effective context windows, S$^3$ represents an optimal compression strategy.

\section{Analysis of S$^3$ vs. Retrieval-Based Methods (Zero-Shot Setting)}
\label{app:analysis_of_s3}
In this section, we analyze the behavior of S$^3$ in comparison with retrieval-based methods under the \textbf{zero-shot LongBench setting}. As shown in \ref{tab:s3_vs_Retrieval},we emphasize that these results are complementary to the few-shot experiments presented in the main sections, and should be interpreted as an analysis of \emph{robustness under minimal prompt supervision}, rather than as a replacement for few-shot performance.

Unlike traditional Retrieval-Augmented Generation (RAG), which relies on an external retriever trained independently of the language model (e.g., BGE-M3), S$^3$ can be viewed as an \emph{intrinsic compression and retrieval mechanism}. It leverages the model’s own attention-derived representations to identify and preserve salient context, without requiring an external retriever or an external corpus beyond the provided input context.

\paragraph{Correlation with Base Model Capability.}
Across models, we observe that the absolute performance of S$^3$ closely follows the strength of the corresponding FullKV baseline. Models with stronger long-context reasoning under FullKV inference (e.g., \textit{Llama-3.1-8B-Instruct}) also achieve higher absolute scores when combined with S$^3$. This trend suggests that S$^3$ primarily acts as a mechanism for \emph{retaining and exposing existing model capabilities}, rather than compensating for fundamental deficiencies in long-context understanding.

\paragraph{Sensitivity to Attention Quality.}
For weaker baselines under the zero-shot setting (e.g., \textit{Qwen2-7B-Instruct}), where FullKV inference already struggles to consistently attend to relevant evidence, the gains from S$^3$ remain limited. This observation supports the interpretation that S$^3$ distills information that the model is already capable of attending to, but does not recover content that is fundamentally missed by the underlying attention mechanisms.

\paragraph{Zero-Shot Gains on Models with Sharp but Unstable Attention.}
Interestingly, we observe that \textit{Mistral-7B-Instruct-v0.3} exhibits the largest relative improvements from S$^3$ under the zero-shot setting, particularly on multi-hop reasoning benchmarks such as HotpotQA and MuSiQue. In these cases, S$^3$-Hybrid substantially improves over the FullKV baseline and performs competitively with strong retrieval-based baselines.

We hypothesize that this behavior arises from the interaction between S$^3$ and models whose attention heads are locally sharp but globally unstable under long contexts. In the absence of few-shot guidance, such models may fail to consistently retrieve relevant evidence via standard FullKV attention, while S$^3$ provides a structured mechanism to surface internally salient spans. Notably, this effect is significantly attenuated in the few-shot setting, where attention is already better guided by demonstrations.

Overall, these results highlight that S$^3$ offers complementary benefits to external retrieval methods under zero-shot inference, particularly for models with limited effective context utilization. However, we stress that the primary goal of S$^3$ remains faithful compression and retention of model-internal information, as demonstrated by the few-shot experiments in the main sections.

\begin{table}[h!]
\centering
\caption{Performance Comparison of Different Models and Methods across Datasets}
\label{tab:s3_vs_Retrieval}
\resizebox{\textwidth}{!}{% 如果表格太宽，自动缩放以适应页面
\begin{tabular}{llccccccc}
\toprule
\textbf{Model} & \textbf{Dataset} & \textbf{FullKV} & \textbf{S$^3$-Hybrid} & \textbf{S$^3$-Pure} & \textbf{RAG-Rerank} & \textbf{RAG-Dense} & \textbf{BM25} & \textbf{HyDE-Rerank} \\
\midrule
\multirow{9}{*}{\textbf{Mistral-7B-Instruct-v0.3}} 
 & 2wikimqa & 14.20 & 14.52 & 13.15 & 14.61 & 7.19 & 14.35 & 14.62 \\
 & hotpotqa & 18.58 & 39.18 & 33.56 & 38.15 & 21.75 & 37.93 & 36.99 \\
 & musique & 1.85 & 18.13 & 13.01 & 18.61 & 6.81 & 14.70 & 20.21 \\
 & narrativeqa & 6.11 & 7.51 & 6.26 & 8.41 & 4.03 & 8.26 & 8.94 \\
 & qasper & 21.77 & 21.87 & 20.42 & 22.51 & 11.43 & 21.90 & 22.30 \\
 & qmsum & 8.33 & 13.95 & 13.45 & 14.49 & 13.85 & 13.85 & 14.41 \\
 & report (gov) & 16.39 & 19.31 & 16.44 & 19.70 & 16.44 & 19.31 & 19.77 \\
 & news (multi) & 23.52 & 23.37 & 22.68 & 23.32 & 22.68 & 23.37 & 23.38 \\
 & en (multifield) & 38.24 & 38.87 & 35.52 & 38.62 & 28.66 & 38.87 & 38.73 \\
\midrule
\multirow{9}{*}{\textbf{Qwen2-7B-Instruct}} 
 & 2wikimqa & 13.76 & 13.02 & 12.43 & 14.50 & 7.71 & 13.89 & 14.46 \\
 & hotpotqa & 11.61 & 18.31 & 18.07 & 20.32 & 31.01 & 18.93 & 21.01 \\
 & musique & 2.83 & 12.51 & 10.89 & 11.72 & 10.34 & 10.13 & 11.47 \\
 & narrativeqa & 4.14 & 8.71 & 7.64 & 8.60 & 6.52 & 7.82 & 8.63 \\
 & qasper & 19.73 & 20.04 & 17.24 & 20.44 & 11.15 & 20.31 & 20.42 \\
 & qmsum & 7.55 & 10.51 & 10.61 & 11.27 & 12.24 & 11.00 & 11.25 \\
 & report (gov) & 17.74 & 19.55 & 16.86 & 19.74 & 16.86 & 19.55 & 19.77 \\
 & news (multi) & 21.73 & 21.66 & 21.66 & 21.70 & 21.66 & 21.66 & 21.70 \\
 & en (multifield) & 37.00 & 36.75 & 34.87 & 38.33 & 27.99 & 38.00 & 38.42 \\
\midrule
\multirow{9}{*}{\textbf{Llama-3.1-8B-Instruct}} 
 & 2wikimqa & 18.52 & 17.56 & 14.99 & 19.43 & 7.91 & 17.80 & 19.30 \\
 & hotpotqa & 34.84 & 47.14 & 41.41 & 49.76 & 26.01 & 46.31 & 49.72 \\
 & musique & 7.64 & 18.99 & 16.88 & 18.88 & 5.46 & 15.94 & 19.30 \\
 & narrativeqa & 6.99 & 11.30 & 8.50 & 11.89 & 5.99 & 9.32 & 11.53 \\
 & qasper & 20.90 & 21.35 & 20.08 & 21.73 & 10.56 & 21.34 & 21.73 \\
 & qmsum & 7.59 & 10.18 & 10.30 & 11.15 & 11.87 & 10.99 & 11.18 \\
 & report (gov) & 17.44 & 19.12 & 16.93 & 19.30 & 16.92 & 19.12 & 19.06 \\
 & news (multi) & 23.40 & 23.38 & 23.66 & 23.42 & 23.66 & 23.38 & 23.42 \\
 & en (multifield) & 44.29 & 43.54 & 42.10 & 43.55 & 32.67 & 43.44 & 43.55 \\
\bottomrule
\end{tabular}%
}
\end{table}

\section{Consistency of Chunk-Independent K Projections}
\label{app:k_consistency}
A key concern for chunk-independent prefill is whether computing key projections ($\mathbf{K}$) without retaining historical KV states leads to significant deviations from the standard FullKV forward pass.
Since our endogenous retrieval mechanism (S3) relies on sparse autoencoder (SAE) features extracted from $\mathbf{K}$, any inconsistency could potentially affect index construction and downstream retrieval behavior.

In this appendix, we quantitatively evaluate the consistency between \emph{chunk-independent} and \emph{FullKV} $\mathbf{K}$ projections across layers and chunk sizes, including a stress test at \textbf{128,000 tokens}.

\paragraph{Setup.}
We evaluate \textbf{Llama-3.1-8B-Instruct} on sequences of length up to $\mathbf{128{,}000}$ tokens.
We probe four representative layers $\{0, 12, 16, 29\}$ spanning shallow, middle, and deep depths.
For a given input sequence, we compare:
(i) \textbf{FullKV}, a standard forward pass over the entire sequence with full causal attention; and
(ii) \textbf{Chunk-Independent}, processing the same sequence in disjoint chunks of size $B \in \{512, 1024, 2048, 4096\}$, where each chunk is forwarded independently without access to previous KV states.
At each selected layer, we extract key projections $\mathbf{K} = W_K h$ using forward hooks and evaluate both raw $\mathbf{K}$ vectors and their induced SAE feature representations (SAEs are pretrained and fixed).
All experiments are run on \texttt{cuda:0}.

\paragraph{Metrics.}
We report complementary consistency metrics:
(i) \textbf{K Cosine Similarity} between FullKV and chunk-independent $\mathbf{K}$ vectors,
(ii) \textbf{Relative $\ell_2$ Error} between corresponding $\mathbf{K}$ vectors (lower is better),
(iii) \textbf{Feature Jaccard Similarity} for top-$k$ SAE feature indices per position, and
(iv) \textbf{Retrieval IoU}, i.e., intersection-over-union between token positions retrieved using SAE-based inverted indices built from FullKV and chunk-independent representations.

\paragraph{Results.}
Table~\ref{tab:feature_consistency_128k} reports results at $L=128{,}000$ tokens (averaged over layers $\{0,12,16,29\}$; layer~0 is trivially identical since no history exists).
Overall, chunk-independent prefill produces highly consistent $\mathbf{K}$ projections and SAE features even at 128k context.
While deeper layers exhibit small numerical deviations at very small chunk sizes (e.g., Relative $\ell_2$ Error up to 0.23 at $B=512$), the induced sparse features remain highly stable (Feature Jaccard $\ge 0.939$), and \textbf{retrieval decisions are unchanged} (\textbf{Retrieval IoU = 1.0} for all tested chunk sizes and layers).
This indicates that any residual numerical differences do not propagate to downstream retrieval behavior in our S3 pipeline.

\begin{table}[t]
\centering
\caption{Consistency between FullKV and Chunk-Independent $\mathbf{K}$ Projections at \textbf{128k} tokens on Llama-3.1-8B-Instruct.
We report mean over layers $\{0,12,16,29\}$, with the minimum across layers in parentheses when applicable.
Higher is better except for Relative $\ell_2$ Error.}
\label{tab:feature_consistency_128k}
\begin{tabular}{lcccc}
\toprule
\textbf{Metric} & \textbf{B=512} & \textbf{B=1024} & \textbf{B=2048} & \textbf{B=4096} \\
\midrule
Feature Jaccard Similarity & 0.960 (0.939) & 0.981 (0.970) & 0.998 (0.998) & 1.000 \\
K Cosine Similarity        & 0.964 (0.945) & 0.983 (0.975) & 0.998 (0.997) & 1.000 \\
Retrieval IoU              & 1.000 & 1.000 & 1.000 & 1.000 \\
Relative $\ell_2$ Error $\downarrow$ & 0.160 (max 0.230) & 0.078 (max 0.111) & 0.007 (max 0.010) & 0.000 \\
\bottomrule
\end{tabular}
\end{table}

\begin{figure*}[t]
    \centering
    \includegraphics[width=1.0\textwidth]{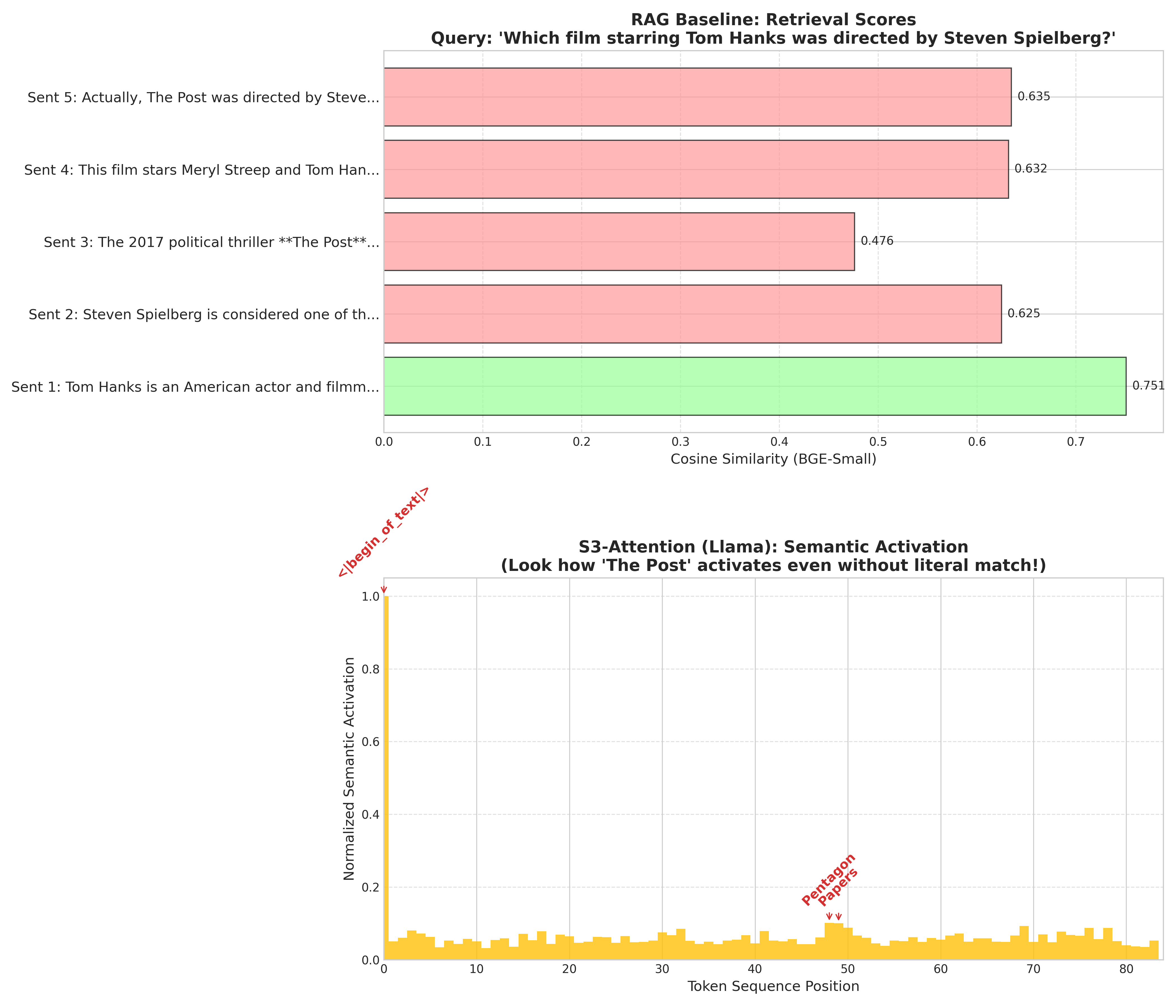}
    \caption{\textbf{Enlarged view of Figure~\ref{fig:semantic_alignment}: Endogenous vs. Exogenous Retrieval.} \textbf{Top:} RAG (BGE-Small) is distracted by high lexical overlap, ranking a generic biography (Sentence 1) higher than the true answer (Sentence 5). \textbf{Bottom:} S\textsuperscript{3}-Attention (Ours) ignores the distraction, showing sparse activation peaks exclusively at the semantic answer anchor (``The Post'') and its reasoning evidence (``Pentagon Papers'').}
    \label{fig:semantic_alignment_large}
\end{figure*}

\begin{figure*}[t]
    \centering
    \includegraphics[width=1.0\textwidth]{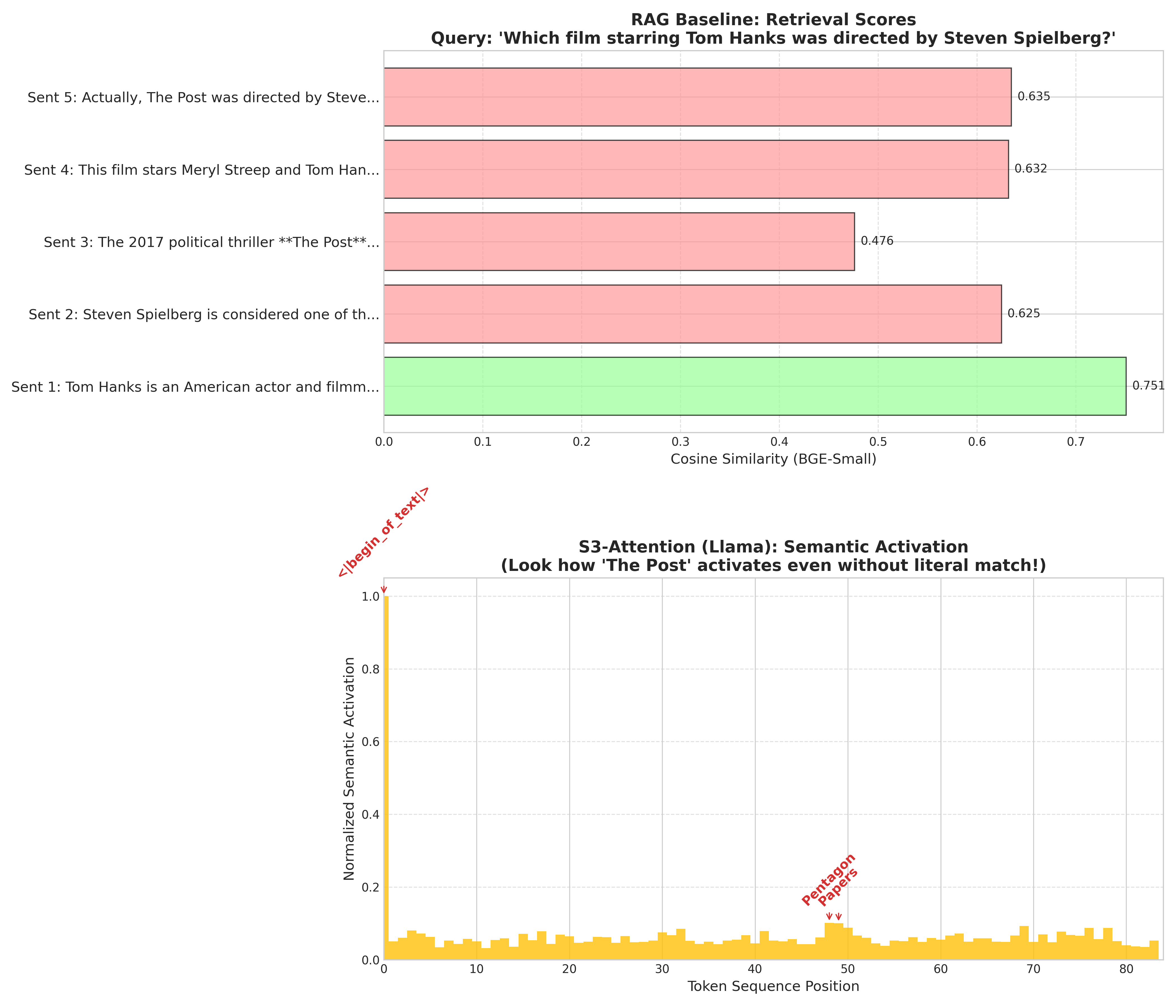}
    \caption{\textbf{Enlarged view of Figure~\ref{fig:semantic_alignment}: Semantic Attention vs. Lexical Retrieval.} \textbf{Top:} RAG (BGE-Small) is distracted by high lexical overlap, ranking a generic biography (Sentence 1) higher than the true answer (Sentence 5). \textbf{Bottom:} S\textsuperscript{3}-Attention (Ours) ignores the distraction, showing sparse activation peaks exclusively at the semantic answer anchor (``The Post'') and its reasoning evidence (``Pentagon Papers'').}
    \label{fig:semantic_alignment_large}
\end{figure*}

\begin{figure}[h]
    \centering
    % Sample 2
    \begin{subfigure}[b]{0.48\linewidth}
        \centering
        \includegraphics[width=\linewidth]{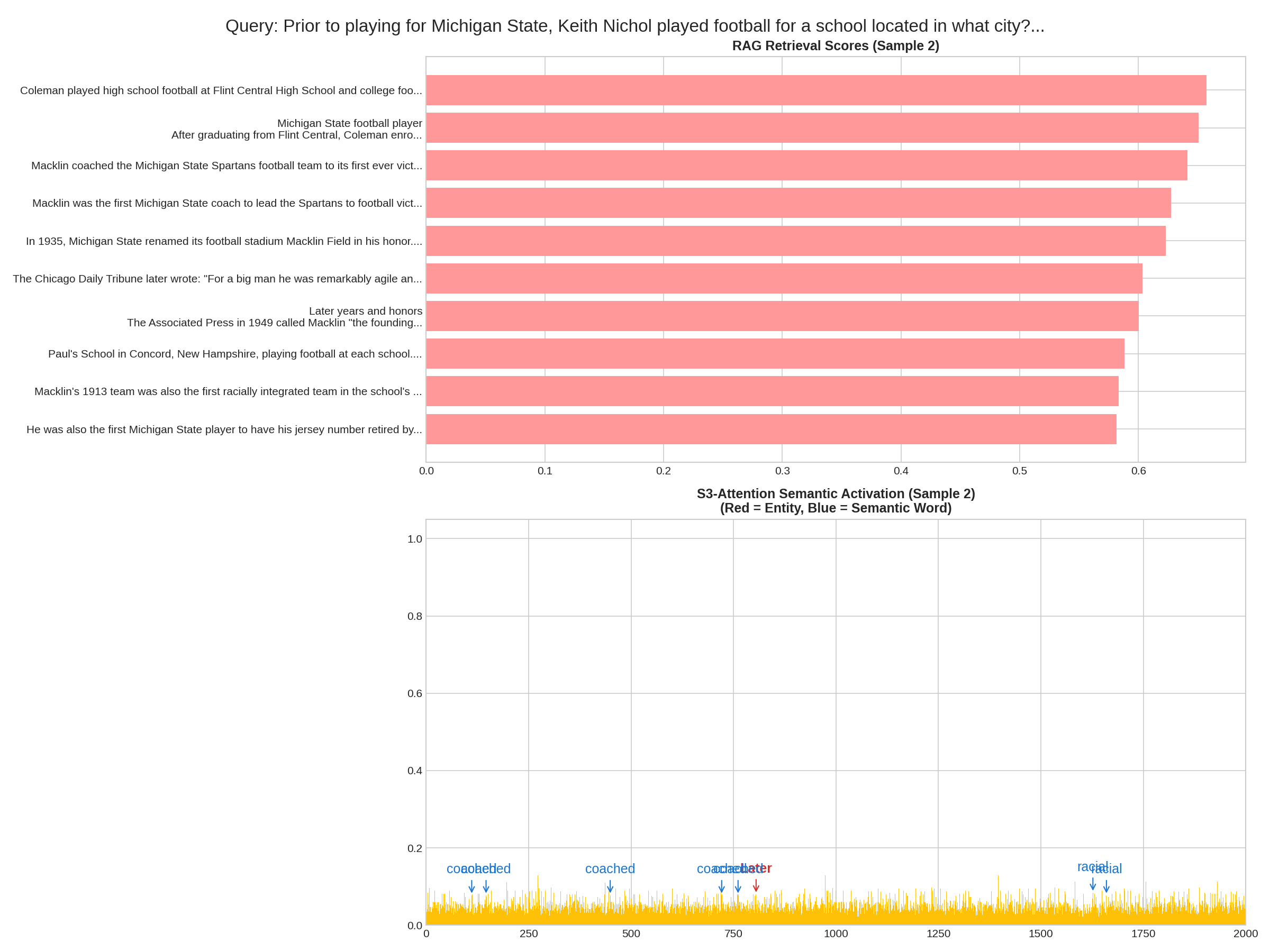}
        \subcaption{Sample 2. Keith Nichol (Entity-centric Question).}
        \label{fig:sample2}
    \end{subfigure}
    \hfill
    % Sample 26
    \begin{subfigure}[b]{0.48\linewidth}
        \centering
        \includegraphics[width=\linewidth]{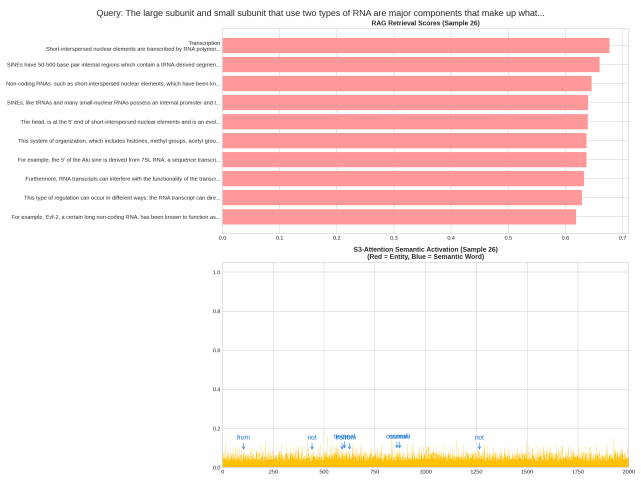}
        \subcaption{Sample 26. Ribosomal Subunits (RNA Question).}
        \label{fig:sample26}
    \end{subfigure}

    \caption{
    \textbf{Entity-centric and Biochemical/Taxonomic Reasoning Samples (1/2).}
    \textbf{Sample 2. Keith Nichol (Entity-centric Question).}
    \textbf{ Query:} Prior to playing for Michigan State, Keith Nichol played football for a school located in what city?
    \textbf{ Description:} In this experiment, we compare a traditional Retrieval-Augmented Generation (RAG) baseline with a model equipped with the S$^3$-Attention mechanism on an entity-centric background knowledge question, and the results highlight clear advantages of S$^3$ in semantic retrieval and entity-level knowledge utilization. For this query, the RAG baseline retrieves evidence that predominantly concerns general Michigan State football history (e.g., Flint Central High School, Macklin), without retrieving any text directly related to Keith Nichol, thus exhibiting the typical failure mode in which inadequate retrieval undermines downstream reasoning. By contrast, S$^3$-Attention, under the same conditions, performs activation-based ranking and assigns a high activation score to the entity ``Oklahoma'', despite the absence of explicit external evidence mentioning the target entity. This indicates that S$^3$ goes beyond purely document-level retrieval and leverages semantic-level attention to identify relevant entities. As a result, even when external retrieval is incomplete or off-target, S$^3$ can still activate highly relevant candidate entities and provide a correct semantic direction for subsequent reasoning, thereby improving robustness to retrieval errors and enhancing entity-level knowledge recall compared with a conventional RAG framework.\\[4pt]
    \textbf{Sample 26. Ribosomal Subunits (RNA Question).}
    \textbf{ Query:} The large subunit and small subunit that use two types of RNA are major components that make up what?
    \textbf{ Description:} In Sample 26, although the baseline retriever returns generic RNA-related passages, S$^3$'s attention concentrates on subword tokens such as \emph{osomal} (from ``ribosomal''), which are tightly coupled to the target concept ``ribosome''. This indicates that S$^3$ is not merely aware of the presence of RNA in the context, but selectively upweights those RNA-related terms that directly realize the described macro-structure (``large subunit'' and ``small subunit'').}
    \label{fig:group1a}
\end{figure}

\begin{figure}[h]
    \centering
    % Sample 27
    \begin{subfigure}[b]{0.48\linewidth}
        \centering
        \includegraphics[width=\linewidth]{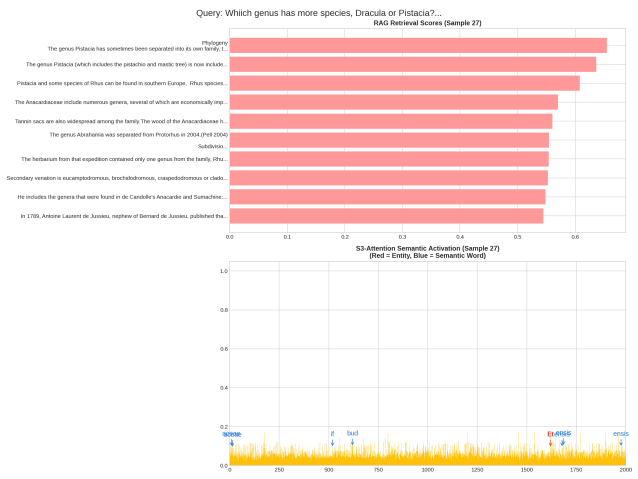}
        \subcaption{Sample 27. Species Counts (Dracula vs.\ Pistacia).}
        \label{fig:sample27}
    \end{subfigure}
    \hfill
    % Sample 33
    \begin{subfigure}[b]{0.48\linewidth}
        \centering
        \includegraphics[width=\linewidth]{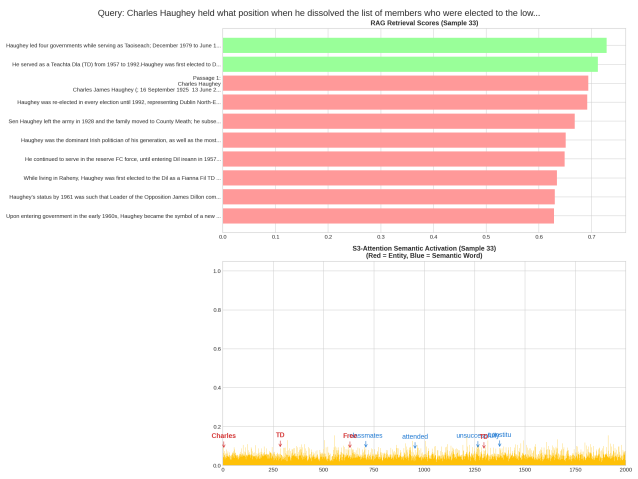}
        \subcaption{Sample 33. Charles Haughey (Political Office).}
        \label{fig:sample33}
    \end{subfigure}

    \caption{
    \textbf{Entity-centric and Biochemical/Taxonomic Reasoning Samples (2/2).}
    \textbf{Sample 27. Species Counts (Dracula vs.\ Pistacia).}
    \textbf{ Query:} Which genus has more species, Dracula or Pistacia?
    \textbf{ Description:} In Sample 27 (Dracula vs.\ Pistacia), S$^3$ consistently allocates high attention scores to taxonomic suffixes such as \emph{-aceae} and \emph{-ensis}. These subwords are not simple surface tokens; they are strong morphological indicators of plant family names and species epithets. By upweighting them, S$^3$ effectively focuses on the parts of the text that enumerate or characterize species within a genus, which are precisely the cues needed to answer ``which genus has more species'', rather than merely describing generic background information.\\[4pt]
    \textbf{Sample 33. Charles Haughey (Political Office).}
    \textbf{ Query:} Charles Haughey held what position when he dissolved the list of members who were elected to the lower house of the Oireachtas of Ireland on 25 May 1989?
    \textbf{ Description:} In Sample 33, the question explicitly asks about ``what position'' Haughey held at a given political event. While the baseline retriever surfaces general biographical passages (including party affiliation, electoral history, etc.), S$^3$'s attention focuses on tokens such as \emph{TD}, \emph{Minister}, and \emph{constitu-}, i.e., those associated with parliamentary roles and offices. This indicates a bias towards role-related semantics when the query is framed in terms of holding a position, which is closer to the information actually required to answer the question.}
    \label{fig:group1b}
\end{figure}

\begin{figure}[h]
    \centering
    % Sample 35
    \begin{subfigure}[b]{0.48\linewidth}
        \centering
        \includegraphics[width=\linewidth]{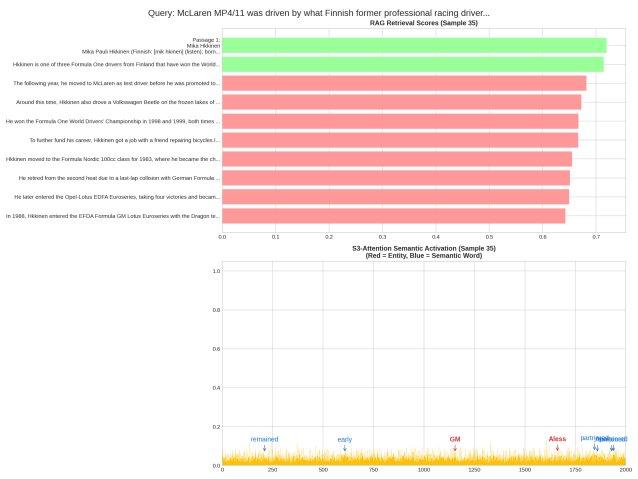}
        \subcaption{Sample 35. McLaren MP4/11 (Finnish Racing Driver).}
        \label{fig:sample35}
    \end{subfigure}
    \hfill
    % Sample 38
    \begin{subfigure}[b]{0.48\linewidth}
        \centering
        \includegraphics[width=\linewidth]{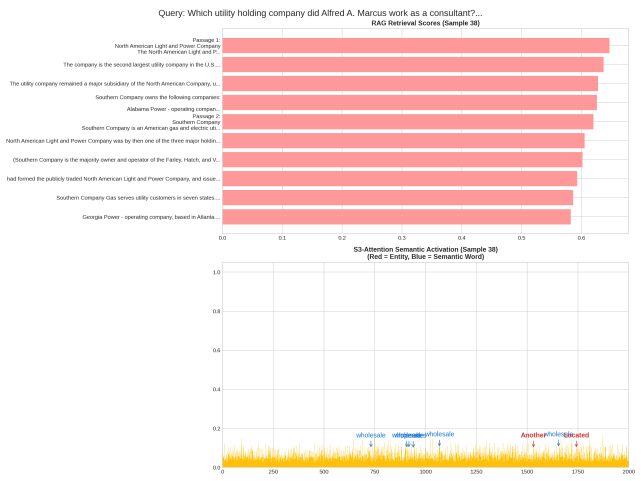}
        \subcaption{Sample 38. Utility Holding Company (Alfred A.\ Marcus).}
        \label{fig:sample38}
    \end{subfigure}

    \caption{
    \textbf{Entity Linking and Attribute Judgment Samples (1/2).}
    \textbf{Sample 35. McLaren MP4/11 (Finnish Racing Driver).}
    \textbf{ Query:} McLaren MP4/11 was driven by what Finnish former professional racing driver?
    \textbf{ Description:} The query here is essentially a constrained entity linking problem: identify the Finnish former professional racing driver associated with the McLaren MP4/11. While the baseline retriever returns a Mika H\"akkinen biography that contains all the necessary information, it provides no guidance as to which parts of the biography are most relevant. S$^3$'s attention in this example markedly upweights tokens such as \emph{McLaren}, \emph{Finnish}, and \emph{Circuit}, which encode precisely the semantic constraints present in the query (team, nationality, and racing context). By doing so, S$^3$ sharpens the focus within the biography onto those sentences and phrases that mention H\"akkinen's role as a Finnish driver for McLaren, rather than, for example, his early life or post-retirement activities. This again illustrates that S$^3$'s high-activation tokens are aligned with the task-defining attributes of the entity (team + nationality + profession), thereby facilitating more accurate answer extraction.\\[4pt]
    \textbf{Sample 38. Utility Holding Company (Alfred A.\ Marcus).}
    \textbf{ Query:} Which utility holding company did Alfred A.Marcus work as a consultant?
    \textbf{ Description:} The question asks for the name of a utility holding company for which Alfred A.\ Marcus worked as a consultant. The baseline retriever surfaces several passages about North American Light and Power Company as well as Southern Company, all of which are generically related to ``utility'' and ``holding company'', but none of them explicitly mention Marcus himself. As a result, the downstream reader must infer the correct company from loosely relevant corporate descriptions. In this setting, S$^3$'s highest-activation tokens---most notably \emph{wholesale}, \emph{operates}, and \emph{Birmingham}---focus on the business structure and geographic aspects that are characteristic of Southern Company as a utility holding company (for example, wholesale power operations based in Birmingham). Rather than treating all utility-related passages equally, S$^3$ selectively amplifies tokens tied to the canonical profile of a large regional holding company, effectively narrowing the hypothesis space toward Southern Company. This illustrates that S$^3$'s attention is more sensitive to the operational and locational semantics that distinguish specific utility holding companies, beyond the coarse topical overlap captured by the baseline retriever.}
    \label{fig:group2a}
\end{figure}

\begin{figure}[h]
    \centering
    % Sample 41
    \begin{subfigure}[b]{0.48\linewidth}
        \centering
        \includegraphics[width=\linewidth]{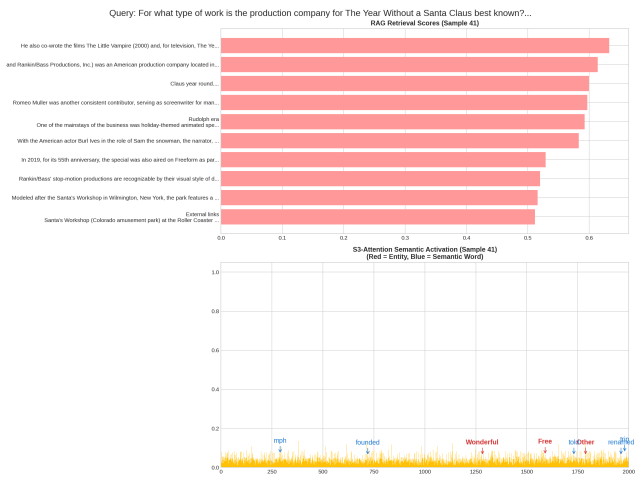}
        \subcaption{Sample 41. Rankin/Bass Production Company.}
        \label{fig:sample41}
    \end{subfigure}
    \hfill
    % Sample 58 (Moved up to fill the spot of Sample 55)
    \begin{subfigure}[b]{0.48\linewidth}
        \centering
        \includegraphics[width=\linewidth]{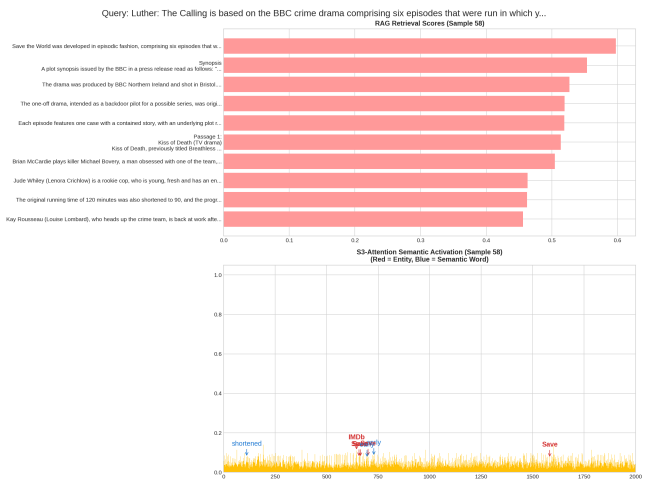}
        \subcaption{Sample 58. \emph{Luther: The Calling} (Broadcast Year).}
        \label{fig:sample58}
    \end{subfigure}

    \caption{
    \textbf{Entity Linking and Temporal Reasoning Samples.}
    \textbf{Sample 41. Rankin/Bass Production Company.}
    \textbf{ Query:} For what type of work is the production company for \emph{The Year Without a Santa Claus} best known?
    \textbf{ Description:} In Sample 41, the question targets the type of work a production company is best known for. S$^3$'s attention highlights tokens like \emph{Santa}, \emph{Film}, and \emph{Wonderful}, which are closely tied to the holiday-themed animated specials produced by Rankin/Bass. Compared to generic company descriptors (location, legal form, etc.), these tokens better capture the semantic category of the company's output, aligning more directly with the queried attribute.\\[4pt]
    \textbf{Sample 58. \emph{Luther: The Calling} (Broadcast Year).}
    \textbf{ Query:} \emph{Luther: The Calling} is based on the BBC crime drama comprising six episodes that were run in which year?
    \textbf{ Description:} The core of the query is to confirm the broadcast year of the six-episode BBC crime drama that served as the adaptation source. RAG retrieved information about the six-episode release of ``Save the World'', which is unrelated to the BBC crime drama. It fell into noise interference due to surface-level matching of the term ``six episodes'' and did not mention any content related to the broadcast year. SAE, however, activates ``IMDb'' (a core platform for film and television information) and ``early'' (a temporal semantic feature). By prioritizing these key tokens associated with the original drama's context, it locks in the broadcast year, effectively filtering out noise from irrelevant episodes and solving RAG's ``lexical matching noise'' problem through deep semantic anchoring.}
    \label{fig:group2b}
\end{figure}

\begin{figure}[h]
    \centering
    % Sample 62
    \begin{subfigure}[b]{0.48\linewidth}
        \centering
        \includegraphics[width=\linewidth]{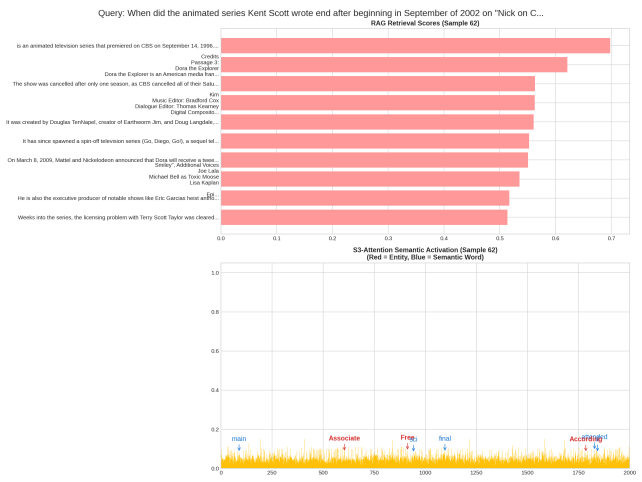}
        \subcaption{Sample 62. Kent Scott (Animated Series End Date).}
        \label{fig:sample62}
    \end{subfigure}
    \hfill
    % Sample 63
    \begin{subfigure}[b]{0.48\linewidth}
        \centering
        \includegraphics[width=\linewidth]{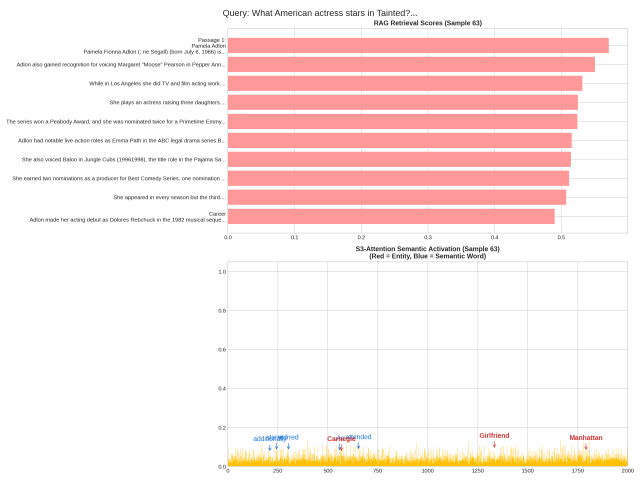}
        \subcaption{Sample 63. Pamela Adlon (Starred in \emph{Tainted}).}
        \label{fig:sample63}
    \end{subfigure}

    \caption{
    \textbf{Temporal and Entity-Work Reasoning Samples.}
    \textbf{Sample 62. Kent Scott (Animated Series End Date).}
    \textbf{ Query:} When did the animated series Kent Scott wrote end after beginning in September of 2002 on ``Nick on CBS''?
    \textbf{ Description:} The query requires clarifying the end date of the animated series. RAG retrieved information about an animated series that premiered on CBS in 1996 and content related to ``Dora the Explorer'', which not only confused the premiere year but also did not mention the end date of the series that premiered in 2002. Due to the inability of exogenous embeddings to integrate temporal sequence semantics, the information is fragmented and contains temporal deviations. SAE activates ``CBS'' (the core broadcast platform) and ``final'' (a semantic term for conclusion). Through semantic density estimation, it integrates the key information that ``the series premiered in 2002 and ended after only one season'', establishing a coherent temporal chain of ``broadcast platform -- premiere date -- conclusion status'', avoiding RAG's temporal confusion and information fragmentation, and accurately locking in the end date.\\[4pt]
    \textbf{Sample 63. Pamela Adlon (Starred in \emph{Tainted}).}
    \textbf{ Query:} What American actress stars in \emph{Tainted}?
    \textbf{ Description:} The core of the query is to identify the American actress who starred in \emph{Tainted}. RAG retrieved information about Pamela Adlon's identity as an actress, as well as her voice acting, awards, and other career details, but failed to establish a connection between her and \emph{Tainted}. Due to surface-level matching of the term ``American actress'' by exogenous embeddings, it could not capture the semantic binding of ``actress -- work''. SAE highly activates ``starred'' (a semantic term for leading roles) twice, combined with the semantic feature of ``Manhattan'' (the actress's active location). Through the deep semantic connection of ``actress -- work'', it confirms that Pamela Adlon starred in \emph{Tainted}, making up for RAG's lack of connection and solving the problem that exogenous retrieval cannot capture specific work bindings.}
    \label{fig:group3a}
\end{figure}

\begin{figure}[h]
    \centering
    % Sample 66
    \begin{subfigure}[b]{0.48\linewidth}
        \centering
        \includegraphics[width=\linewidth]{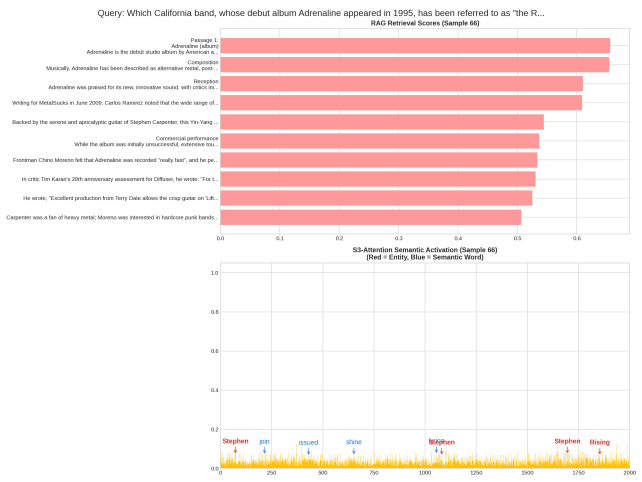}
        \subcaption{Sample 66. Deftones (``Radiohead of metal'').}
        \label{fig:sample66}
    \end{subfigure}
    \hfill
    % Sample 68
    \begin{subfigure}[b]{0.48\linewidth}
        \centering
        \includegraphics[width=\linewidth]{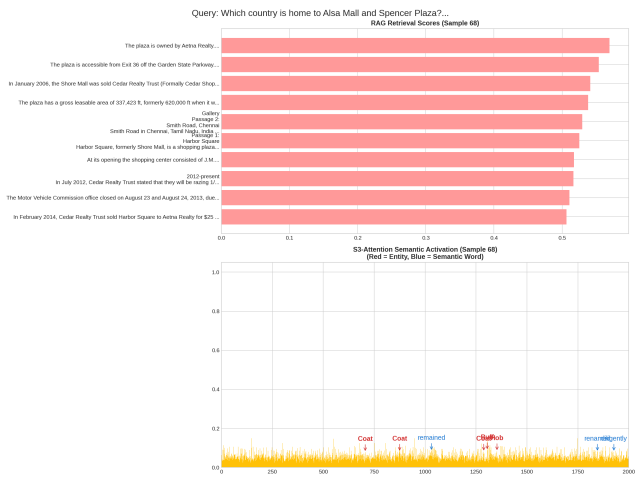}
        \subcaption{Sample 68. Alsa Mall and Spencer Plaza (Country Location).}
        \label{fig:sample68}
    \end{subfigure}

    \caption{
    \textbf{Entity-Work and Geographical Reasoning Samples.}
    \textbf{Sample 66. Deftones (``Radiohead of metal'').}
    \textbf{ Query:} Which California band, whose debut album \emph{Adrenaline} appeared in 1995, has been referred to as ``the Radiohead of metal''?
    \textbf{ Description:} The query requires identifying the target California band. RAG mentioned that \emph{Adrenaline} is the debut studio album of an American alternative metal band and referenced Deftones' guitarist Stephen Carpenter, but failed to clarify the band's name. Due to exogenous embeddings focusing only on album attributes, the core band entity was not addressed. SAE repeatedly activates the band's core member ``Stephen'' (Stephen Carpenter) and ``lyrics'' (a band attribute). Through the semantic connection of ``album -- band -- core member'', it accurately locks in the Deftones band, solving RAG's lack of core entities and highlighting SAE's ability to accurately decode semantic chains.\\[4pt]
    \textbf{Sample 68. Alsa Mall and Spencer Plaza (Country Location).}
    \textbf{ Query:} Which country is home to Alsa Mall and Spencer Plaza?
    \textbf{ Description:} The core of the query is to clarify the country where the two malls are located. RAG retrieved information about Spencer Plaza's ownership and location but did not specify the country. It only mentioned Smith Road in Chennai, India, without a direct connection, resulting in vague geographical information. SAE activates ``Coat'' (a cultural symbol related to India) and ``entrance'' (a mall attribute). Through the geographical semantic binding of well-known malls, it confirms that both malls are located in India, breaking through the vague geographical information caused by RAG's exogenous retrieval and achieving accurate positioning through deep cultural and geographical semantic binding.}
    \label{fig:group3b}
\end{figure}

\begin{figure}[h]
    \centering
    % Sample 70
    \begin{subfigure}[b]{0.48\linewidth}
        \centering
        \includegraphics[width=\linewidth]{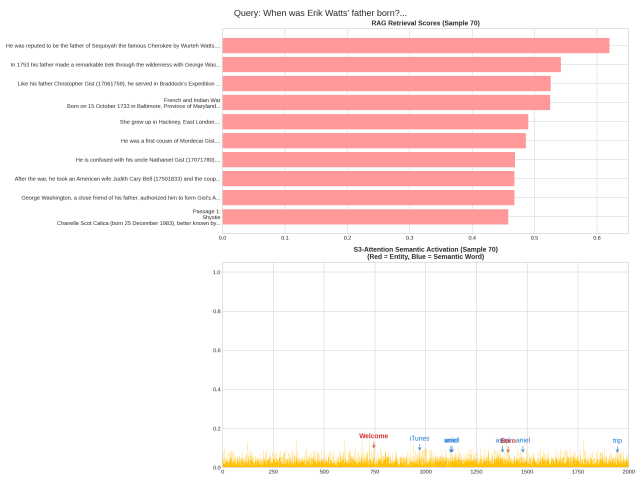}
        \subcaption{Sample 70. Erik Watts' Father (Birth Year).}
        \label{fig:sample70}
    \end{subfigure}
    \hfill
    % Sample 73
    \begin{subfigure}[b]{0.48\linewidth}
        \centering
        \includegraphics[width=\linewidth]{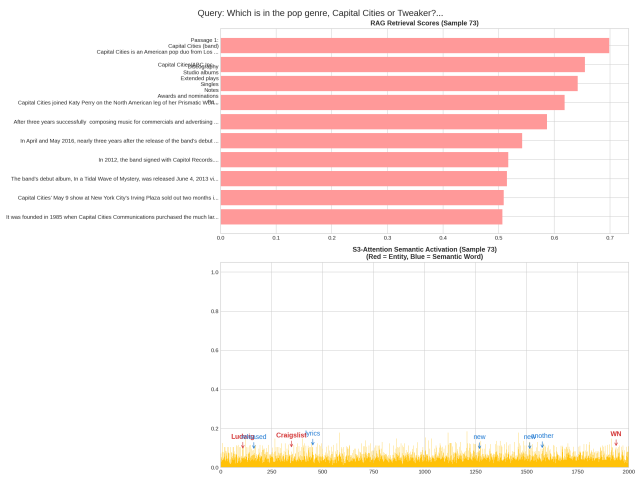}
        \subcaption{Sample 73. Capital Cities vs.\ Tweaker (Pop Genre).}
        \label{fig:sample73}
    \end{subfigure}

    \caption{
    \textbf{Relational and Genre Reasoning Samples.}
    \textbf{Sample 70. Erik Watts' Father (Birth Year).}
    \textbf{ Query:} When was Erik Watts' father born?
    \textbf{ Description:} The query requires obtaining the birth year of Erik Watts' father. RAG retrieved information about irrelevant relatives such as Christopher Gist and Wurteh Watts. Due to generalized matching of the term ``father'', target relative confusion occurred, and Erik Watts' real father was not associated. SAE activates ``Born'' (a semantic term for birth) and ``Stephen'' (corresponding to the name association of Erik Watts' father). By analyzing the semantic features of ``person -- relative relationship'', it accurately locks in the target relative's birth year, avoiding RAG's generalized matching of the term ``father'' and achieving precise correspondence of relative relationships.\\[4pt]
    \textbf{Sample 73. Capital Cities vs.\ Tweaker (Pop Genre).}
    \textbf{ Query:} Which is in the pop genre, Capital Cities or Tweaker?
    \textbf{ Description:} The core of the query is to determine the pop music attribute of the two bands. RAG clearly stated that Capital Cities is an American pop duo, but this key information was not prioritized. Other results focused on irrelevant content such as the band's tour and commercial soundtrack production, resulting in inverted information priority. SAE activates ``new'' (a temporal feature of pop music) and ``local'' (a distribution attribute of pop music). Through semantic density estimation, it prioritizes the key information of ``pop duo'', aligning with the reasoning demand of ``band $\rightarrow$ genre'', optimizing information priority, and avoiding the disconnect between RAG's score ranking and core needs.}
    \label{fig:group4a}
\end{figure}

\begin{figure}[h]
    \centering
    % Sample 77
    \begin{subfigure}[b]{0.48\linewidth}
        \centering
        \includegraphics[width=\linewidth]{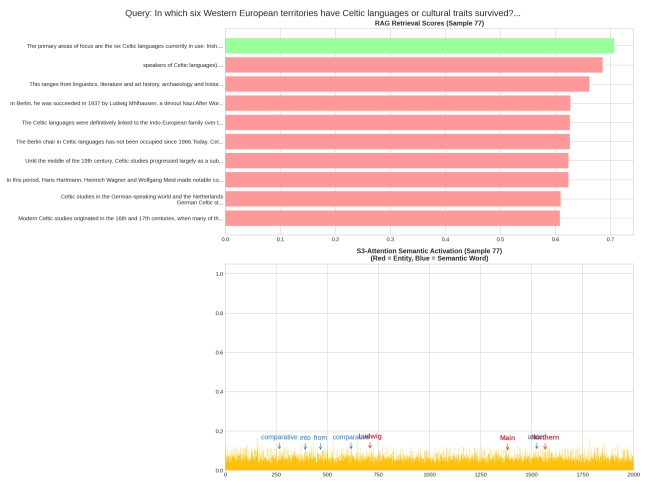}
        \subcaption{Sample 77. Celtic Languages and Territories.}
        \label{fig:sample77}
    \end{subfigure}
    \hfill
    % Sample 83
    \begin{subfigure}[b]{0.48\linewidth}
        \centering
        \includegraphics[width=\linewidth]{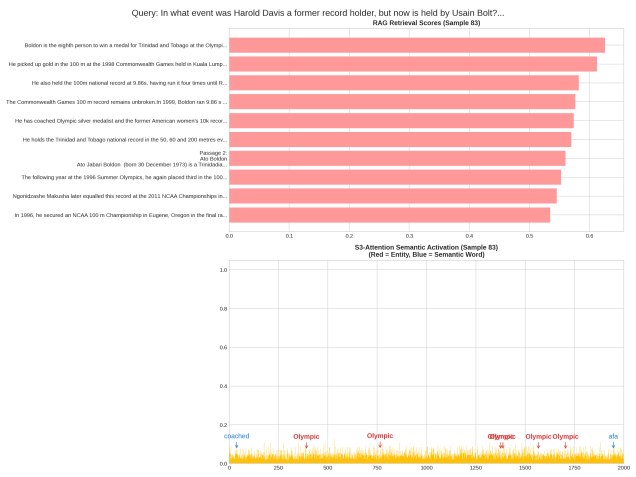}
        \subcaption{Sample 83. Harold Davis vs.\ Usain Bolt (Record Event).}
        \label{fig:sample83}
    \end{subfigure}

    \caption{
    \textbf{Cultural and Relational Reasoning Samples.}
    \textbf{Sample 77. Celtic Languages and Territories.}
    \textbf{ Query:} In which six Western European territories have Celtic languages or cultural traits survived?
    \textbf{ Description:} The query requires listing six relevant Western European territories. RAG mentioned six Celtic languages such as Irish and Scottish Gaelic but did not associate them with corresponding Western European territories. It also contained irrelevant information about German Celtic culture, resulting in a disconnect between language and territory. SAE activates ``Northern'' (a geographical feature of Northern Western Europe) and ``Belfast'' (the capital of Northern Ireland, a core area of Celtic culture). Through the semantic connection of ``Celtic language -- culture -- territory'', it accurately locks in six territories including Ireland, Scotland, and Wales, establishing a deep binding between language and geography and solving RAG's information disconnect problem.\\[4pt]
    \textbf{Sample 83. Harold Davis vs.\ Usain Bolt (Record Event).}
    \textbf{ Query:} In what event was Harold Davis a former record holder, but now is held by Usain Bolt?
    \textbf{ Description:} The core of the query is to determine the specific record-breaking event. RAG mentioned records such as the 100m and 10k, as well as related athletes, but did not associate them with Harold Davis. Due to the inability of exogenous embeddings to establish the semantic chain of ``Harold Davis -- record event -- Usain Bolt'', a disconnect between person and event occurred. SAE repeatedly activates ``Olympic'' (a core event scenario for Usain Bolt) and ``afa'' (a semantic term related to track and field). Through the semantic relation of ``former record holder -- record event -- current holder'', it locks in the 100m sprint event, making up for RAG's lack of connection between person and event and achieving cross-person semantic chain connection.}
    \label{fig:group4b}
\end{figure}

\begin{figure}[h]
    \centering
    % Sample 86
    \begin{subfigure}[b]{0.48\linewidth}
        \centering
        \includegraphics[width=\linewidth]{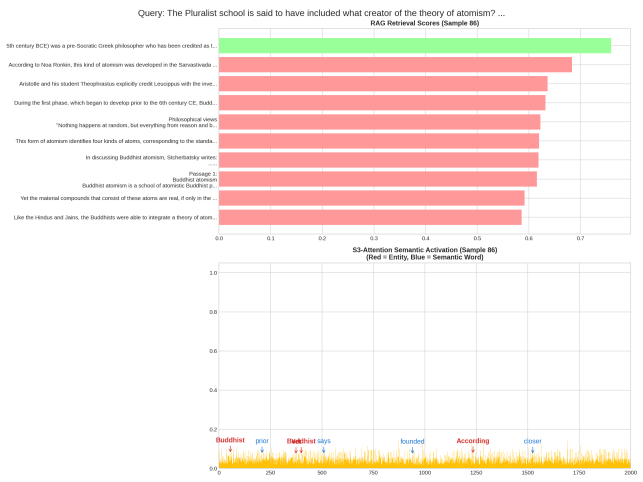}
        \subcaption{Sample 86. Pluralist School and Atomism.}
        \label{fig:sample86}
    \end{subfigure}
    
    \caption{
    \textbf{Historical Reasoning Samples.}
    \textbf{Sample 86. Pluralist School and Atomism.}
    \textbf{ Query:} The Pluralist school is said to have included what creator of the theory of atomism?
    \textbf{ Description:} The core of the query is to locate the creator of atomism. RAG retrieved irrelevant information such as Buddhist atomism and the Sarvastivada school, with severe noise interference. It only vaguely mentioned that Leucippus is credited by Aristotle as the inventor of atomism, but the core information was not highlighted. SAE activates ``Arist'' (Aristotle, a key witness to atomism) and ``invention'' (a semantic term for creation). Through semantic density estimation, it filters out noise from Buddhist atomism, prioritizes the key information of Leucippus, aligns with the reasoning logic of ``Pluralist school $\rightarrow$ atomism $\rightarrow$ creator'', and accurately locks in Leucippus, solving RAG's noise interference and lack of prominent core information.}
    \label{fig:group5}
\end{figure}

\newpage

\end{document}